\documentclass[]{bytedance}

\usepackage[toc,page,header]{appendix}

\usepackage{natbib}

\usepackage{CJKutf8}

\usepackage{xargs}  

\usepackage{todonotes}  

\usepackage{multirow}

\usepackage{cleveref}

\usepackage{amsmath}
\usepackage{dsfont}

\usepackage{subcaption}

\usepackage{svg}

\usepackage{mathrsfs}
\usepackage{adjustbox}
\usepackage{multirow}
\usepackage{multirow}
\usepackage{multicol}
\usepackage{tcolorbox}
\usepackage{changepage}
\usepackage{graphicx}
\usepackage{amssymb}
\usepackage{array}
\usepackage{bm}
\usepackage{colortbl}
\usepackage{xcolor}
\usepackage{amsthm}
\usepackage{thmtools}
\usepackage{longtable}
\usepackage{booktabs}
\usepackage{array}
\usepackage{arydshln}
\usepackage{minitoc}

\usepackage{pifont}
\definecolor{mygreen}{RGB}{112, 180, 143}
\definecolor{myred}{RGB}{242, 128, 128}
\title{VAR RL Done Right: Tackling Asynchronous Policy Conflicts in Visual Autoregressive Generation}

\author{
    Shikun Sun$^{1,2}$, ~~Liao Qu$^2$, 
    ~~Huichao Zhang$^{2,\dagger}$, ~~Yiheng Liu$^2$, ~~Yangyang Song$^2$, ~~Xian Li$^{2}$, ~~Xu Wang$^{2}$,  
     ~~Yi Jiang$^{2}$, ~~Daniel K. Du$^2$, ~~Xinglong Wu$^2$, ~~Jia Jia$^{1, \dagger}$}
\vspace{3em}
\affiliation[1]{Tsinghua University}
\affiliation[2]{ByteDance}
\contribution[\dagger]{Corresponding Authors\vspace{1em}}
\abstract{
Visual generation is dominated by three paradigms: AutoRegressive (AR), diffusion, and Visual AutoRegressive (VAR) models. Unlike AR and diffusion, VARs operate on heterogeneous input structures across their generation steps, which creates severe asynchronous policy conflicts. This issue becomes particularly acute in reinforcement learning (RL) scenarios, leading to unstable training and suboptimal alignment.
To resolve this, we propose a novel framework to enhance Group Relative Policy Optimization (GRPO) by explicitly managing these conflicts. Our method integrates three synergistic components: 1) a stabilizing intermediate reward to guide early-stage generation; 2) a dynamic time-step reweighting scheme for precise credit assignment; and 3) a novel mask propagation algorithm, derived from principles of Reward Feedback Learning (ReFL), designed to isolate optimization effects both spatially and temporally. Our approach demonstrates significant improvements in sample quality and objective alignment over the vanilla GRPO baseline, enabling robust and effective optimization for VAR models.
}

\date{\today}
\checkdata[Official Page]{\url{https://github.com/ByteVisionLab/NextFlow} }
\newcommand{\ALG}{NextFlow-RL}
\newtheorem{definition}{Definition}

\begin{document}
\begin{CJK*}{UTF8}{gbsn}

\maketitle

\definecolor{chinese_red}{HTML}{8B4513}
\definecolor{english_blue}{HTML}{4169E1}

\definecolor{mycolor_blue}{HTML}{E7EFFA}
\definecolor{mycolor_green}{HTML}{E6F8E0}
% \definecolor{mycolor_green}{HTML}{D1E2F4}
\definecolor{mycolor_gray}{HTML}{ECECEC}
\definecolor{pearDark}{HTML}{2980B9}
\definecolor{lightergray}{HTML}{D3D3D3}

%Visual generation is dominated by three paradigms: AutoRegressive (AR), diffusion, and Visual AutoRegressive (VAR) models. Unlike AR and diffusion, VARs operate on heterogeneous input structures across their generation steps, which creates severe asynchronous policy conflicts. This issue becomes particularly acute in reinforcement learning (RL) scenarios, leading to unstable training and suboptimal alignment.
%To resolve this, we propose a novel framework to enhance Group Relative Policy Optimization (GRPO) by explicitly managing these conflicts. Our method integrates three synergistic components: 1) a stabilizing intermediate reward to guide early-stage generation; 2) a dynamic time-step reweighting scheme for precise credit assignment; and 3) a novel mask propagation algorithm, derived from principles of Reward Feedback Learning (ReFL), designed to isolate optimization effects both spatially and temporally. Our approach demonstrates significant improvements in sample quality and objective alignment over the vanilla GRPO baseline, enabling robust and effective optimization for VAR models.

\section{Introduction}
\label{sec:intro}

Recent advances in visual generation are dominated by three paradigms: autoregressive (AR) models~\cite{vqvae,vqvae2,vqgan,yu2022scaling,wang2024emu3}, diffusion models~\cite{Rombach_2022_CVPR,podell2023sdxl,SD3,kolors,chen2024pixart,zheng2024cogview3,li2024playground,comanici2025gemini}, and visual autoregressive (VAR)~\cite{var, han2025infinity, liu2025infinitystar} models. Unlike AR and diffusion, VAR operates over a hierarchy of discrete token grids whose spatial shape changes across resolution levels; at each step, the model emits a \textit{parallel} grid of tokens rather than a single symbol or a fixed-length vector. While this heterogeneous, cross-scale design aligns with modern high-resolution backbones and enables fast synthesis, it also creates severe challenges for RL alignment, where the RL stage operates with far fewer samples than pretraining, exacerbating instability. \textbf{To our knowledge, we are the first to conduct a systematic RL study for text-to-image VAR models}, establishing a practical training recipe and revealing their unique RL failure modes.
The crux of the difficulty is \textit{asynchronous policy conflict} across timesteps. As illustrated in Figure~\ref{fig:motivation}, the number of query tokens can fluctuate by orders of magnitude from coarse to fine scales, inducing large variation in task similarity across steps, leading to instability, slow convergence, and suboptimal alignment when directly applying GRPO to VAR, as shown in training curve in Figure~\ref{fig:motivation_2}, where supervised RL at a partial-prefix scale outperforms its full-scale counterpart.
\begin{figure}[t]
  \centering
  \includegraphics[width=.7\linewidth]{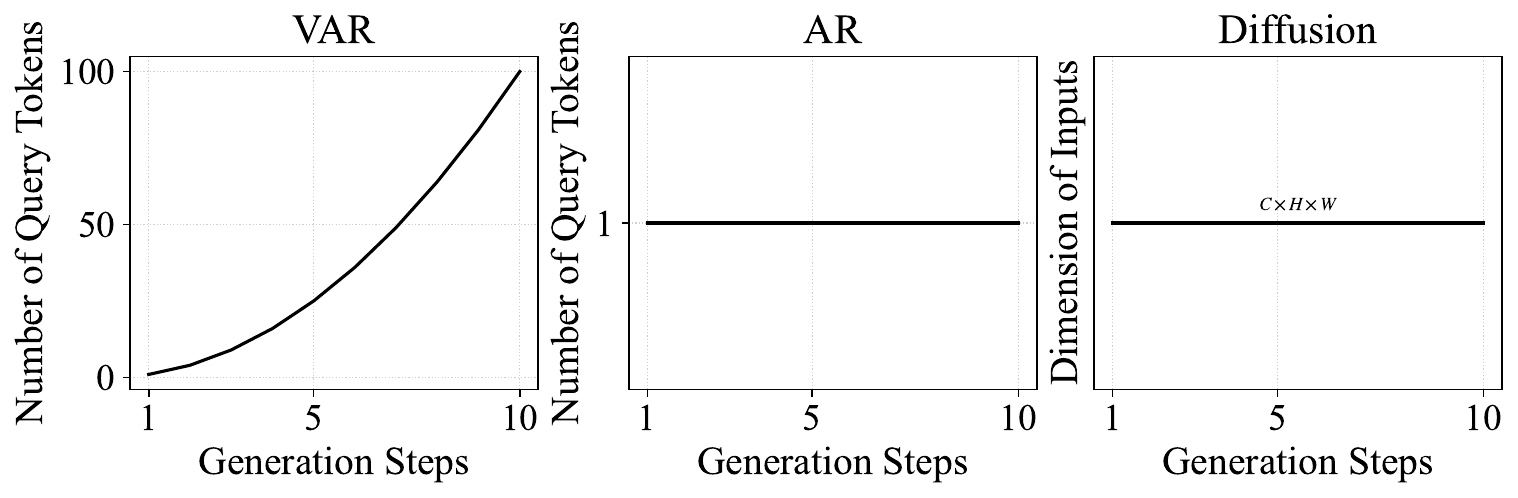}
  \caption{The number of query tokens across different timesteps in VAR generation fluctuates significantly, leading to varying task similarities and potential policy conflicts during RL optimization.}
  \label{fig:motivation}
\end{figure}

% \noindent
% Firstly we formalize VAR generation as a environment-deterministic Markov Decision Process (MDP) over multi-resolution token sequences. In KL-regularized RL~\cite{ziebart2008maximum}, the optimal policy takes an exponentially tilted form that reweights the prior by the exponentiated optimal action-value; under deterministic dynamics, the optimal $Q$ admits a soft-value representation as a log-moment-generating function of the terminal return~\cite{zhou2025q}. However, VAR imposes an additional structural constraint: actions within the same scale are sampled independently across spatial locations conditioned on the current state. Consequently, the unconstrained global optimum of KL-regularized RL may lie \textit{outside} this family, while a \textit{family-optimal} solution exists \textit{within} it and is still reachable by RL. This gap, compounded by asynchronous policy conflicts, motivates a staged RL objective that alleviates multi-step interference without harming the family-optimal policy.
% \shikun{rethink the logic here}

\noindent
To address this core problem, we introduce a simple and effective framework that enhances GRPO~\cite{xue2025dancegrpo,flowgrpo} for VAR with three synergistic components:
(i) a stabilizing intermediate reward that provides dense, low-variance feedback to early steps while preserving family-optimality;
(ii) a dynamic time-step reweighting that normalizes per-step contributions by the number of query tokens, balancing gradients across scales; and
(iii) a mask propagation mechanism, inspired by gradient-based reward-feedback practices in RL for generative models~\cite{xu2023imagereward}, that spatially and temporally isolates optimization effects to the tokens most responsible for the final return.
Concretely, we propose \textbf{Value as Middle Return} (VMR). Motivated by KL-regularized RL, we insert at a middle step $m$ an intermediate soft return and optimize the pre-$m$ and post-$m$ segments with GRPO in a stage-wise fashion. This construction yields more frequent, lower-variance feedback to early decisions and, crucially, does \textit{not} alter the optimal policy within the family: it is structure-preserving reward shaping that leaves the family-optimal solution unchanged while making it easier to reach in practice.
Second, we propose \textbf{Per-Action Normalization Weighting} (PANW). To counteract step-wise heterogeneity, we weight each step-$t$ loss by $k_t = 1/(h_t w_t)$, where $h_t \times w_t$ is the token-grid size at that step, followed by step-level normalization. This balances KL usage and gradient scales across timesteps, mitigating the dominance of high-resolution updates and improving stability.
Third, we introduce \textbf{Mask Propagation} (MP). We maintain a spatiotemporal mask that tracks tokens likely to contribute to the terminal reward and propagates this mask backward along the sequence. The mask gates intermediate rewards and gradients, focusing credit assignment on causally relevant regions and reducing variance across both space and time.
Empirically, our framework consistently stabilizes training and accelerates convergence over vanilla GRPO in this setting, delivering substantial gains in sample quality and objective alignment on text rendering benchmarks~\cite{chen2024textdiffuser,du2025textcrafter,tuo2023anytext}. It also achieves strong improvements over a TokenFlow-T2I starting point~\cite{qu2025tokenflow} and attains state-of-the-art results among diffusion-centric baselines~\cite{Rombach_2022_CVPR,podell2023sdxl,SD3,kolors,chen2024pixart,zheng2024cogview3,li2024playground,comanici2025gemini}. These results underscore that properly structuring the RL objective and balancing updates across heterogeneous steps are crucial for reliable text-to-image alignment.
Our contributions are threefold:
\begin{itemize}
    \item We diagnose asynchronous policy conflicts in RL for VAR and formalize the process as a deterministic MDP, clarifying why bandit-style GRPO becomes unstable under heterogeneous, parallel actions.
    \item We propose VMR, a structure-preserving intermediate reward that provides dense feedback to early steps and provably does not alter the family-optimal policy; together with PANW and MP, it balances step-wise gradients and sharpens spatiotemporal credit assignment.
    \item We present, to our knowledge, the first systematic RL framework for \textit{text-to-image VAR}, demonstrating robust, state-of-the-art improvements over GRPO baseline and diffusion-centric competitors, supported by comprehensive ablations.
\end{itemize}
\section{Related Work}
\begin{figure}[t]
  \centering
  \includegraphics[width=.7\linewidth]{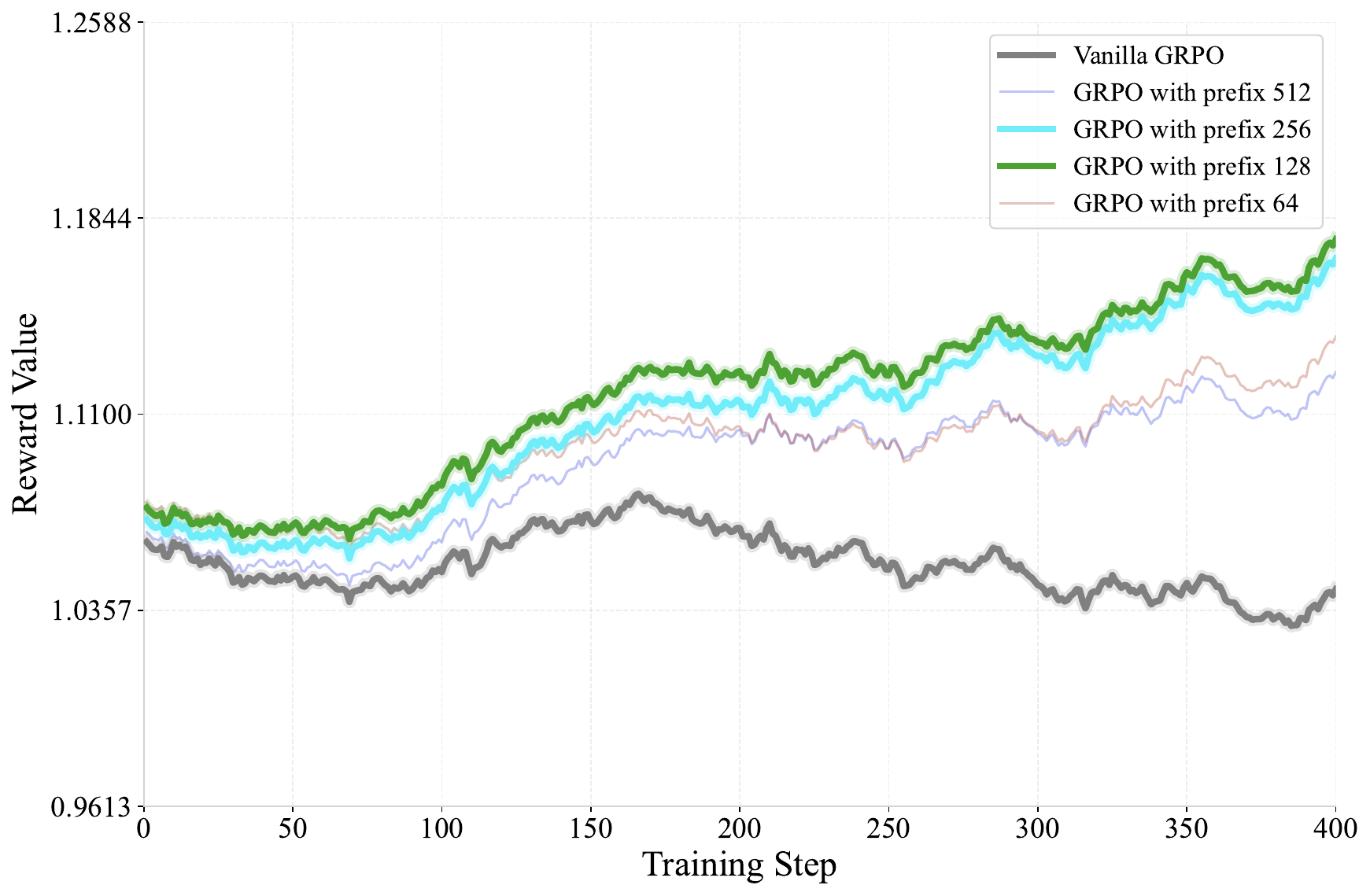}

   \caption{Comparison of training curves between vanilla GRPO and GRPO with VMR across varying prefix scales.}
   \label{fig:motivation_2}
\end{figure}

\subsection{Autoregressive Image Generation}

Autoregressive (AR) image generation has emerged as a compelling paradigm, inspired by the success of next-token prediction in large language models \cite{gpt1,gpt2,gpt3,gpt4}. Early works such as VQ-VAE \cite{vqvae,vqgan} and VQGAN \cite{vqgan} introduced learned codebooks to discretize images into tokens, enabling sequential generation in raster-scan order. This approach has since been widely adopted in text-to-image synthesis \cite{yu2022scaling} and unified multimodal understanding and generation tasks \cite{wu2025janus,chen2025janus,wang2024emu3}. However, raster-scan AR models typically require significantly more sampling steps compared to diffusion-based approaches, limiting their inference-time efficiency.
To address this limitation, VAR \cite{var} proposed a next-scale prediction strategy—a coarse-to-fine framework that predicts tokens across progressively finer spatial scales, aligning more closely with human perceptual hierarchy. This method substantially improves both sampling efficiency and image fidelity. Subsequent studies have further advanced this paradigm \cite{jiao2025flexvar,li2024imagefolder}, demonstrating strong performance in high-resolution text-to-image generation \cite{han2025infinity,qu2025tokenflow,voronov2024switti,tang2024hart,liu2025infinitystar,ma2024star,zhuang2025vargpt}.
Building on these advances, this work investigates the use of reinforcement learning to further enhance VAR-style autoregressive image generation, aiming to improve training stability and overall generative performance.

\subsection{Reinforcement Learning for Image Generation}

Reinforcement Learning (RL) has been widely adopted for enhancing the reasoning capabilities and human alignment of Large Language Models (LLMs) \cite{jaech2024openai,guo2025deepseek}. A significant recent advancement is GRPO \cite{shao2024deepseekmath}, which simplifies policy optimization by eliminating the value model required in PPO \cite{schulman2017proximal}, achieving notable empirical gains through a relative, group-based objective.
This paradigm has been successfully extended to visual generation to improve output fidelity and controllability. For instance, flow-based generative models \cite{flux,SD3} have been aligned with human preferences \cite{kirstain2023pick,wu2023human} and specific prompt following objectives \cite{huang2023t2i,ghosh2023geneval} using GRPO-based techniques \cite{flowgrpo,xue2025dancegrpo}.
The application of RL differs for autoregressive (AR) raster-scan models \cite{argrpo,wang2025simplear,jiang2025t2i,ma2025stage,tong2025delving,zhang2025group}. T2I-R1 \cite{jiang2025t2i} employs RL to bolster semantic reasoning for superior text-to-image alignment, AR-GRPO \cite{argrpo} establishes a direct baseline by integrating GRPO into the AR sampling process, and SimpleAR \cite{wang2025simplear} unifies model training with RL to holistically improve output quality.
Despite these advances, the use of RL in next-scale prediction AR models remains largely unexplored. While \cite{gallici2025fine} applies GRPO to a class-conditioned model \cite{var}, it does not address the core structural challenges introduced by multi-scale parallel token generation. Compared to raster-scan AR or diffusion models, this family of architectures presents unique challenges compared to raster-scan AR or diffusion-based approaches, primarily due to structural differences in policy optimization that arise from multi-scale, parallel token generation. These differences necessitate tailored reward design and optimization strategies that have yet to be thoroughly investigated.
% Value As Return
%       - VAR同scale的独立性导致rl的optimal solution 有特定的结构，in-the-wild optimal solution 在scope外面
%       - scope内的optimal solution还是可以用rl的算法reach
%       - 结构化reward不会harm var scope内的optimal solution
% Reweighting
% Mask Propagation

\section{Methodology}
\label{method}
\subsection{Preliminary}
\subsubsection{Visual AutoRegressive Models}
Visual AutoRegressive (VAR) models generate images by producing a sequence of discrete tokens that represent the image at multiple resolutions. Specifically, an image is represented as a sequence of token grids \(\mathbf{r}_1, \mathbf{r}_2, \ldots, \mathbf{r}_T\), where each \(\mathbf{r}_t\) corresponds to a grid of tokens at resolution level \(t\) shaped \(h_t \times w_t\). The generation process starts from the lowest resolution and progressively refines the image by generating higher-resolution token grids conditioned on the previously generated lower-resolution grids.
% \shikun{maybe encoder decoder process.}
\subsubsection{VAR Sequences as MDPs}
Because we aim to optimize VAR using a Hierarchical Reinforcement Learning~\cite{nachum2018data} framework, unlike the bandit-style setup used in vanilla GRPO, we must explicitly define the MDP \((\mathcal{S}, \mathcal{A}, P, r)\) for VAR sequence generation. We formulate the process as follows (for simplicity, we omit external control inputs such as text or images):

\begin{itemize}
    \item \textbf{Action Space} \(\mathcal{A}\): Each action \(\mathbf{a}^\theta_t (\mathbf{s}_t)= \mathbf{r}_{t+1} \in \mathcal{A}\) corresponds to generating the next resolution tokens grid \(h_{r+1} \times w_{r+1}\).
    \item \textbf{State Space} \(\mathcal{S}\): Each state is the partial VAR sequence \(\mathbf{s}_t = (\mathbf{r}_1, \mathbf{r}_2, \ldots, \mathbf{r}_t) \in \mathcal{S}\).
    \item \textbf{State Transition Probability} \(P\): The transition is deterministic, given by \(P(\mathbf{s}_{t+1} \mid \mathbf{s}_t, \mathbf{a}_t) = \delta_{(\mathbf{r}_1, \mathbf{r}_2, \ldots, \mathbf{r}_{t+1})}\).
    \item \textbf{Reward Function} \(r\): As in typical image generation tasks, the environment only provides a final return \(R(\mathbf{s}_T) = \sum_i r_i\), reflecting evaluation in real-world settings.
\end{itemize}

\subsubsection{The Optimal Solution for KL-Regularized RL}
A well-established result in KL-regularized reinforcement learning is that the optimal policy
\(\pi^*\), which maximizes the expected return under a soft KL constraint, is given by~\cite{ziebart2008maximum}:
\begin{equation}
    \label{general optimal solution}
    \pi^*(\mathbf{a}_t \mid \mathbf{s}_t)
    \propto
    \pi_{\text{old}}(\mathbf{a}_t \mid \mathbf{s}_t)
    \exp\left(\frac{1}{\eta} Q^*(\mathbf{s}_t, \mathbf{a}_t)\right),
\end{equation}
where \(\eta\) is the temperature parameter and \(Q^*(\mathbf{s}_t, \mathbf{a}_t)\) denotes the optimal action-value
function:
\begin{equation}
    Q^*(\mathbf{s}_t, \mathbf{a}_t)
    = \mathbb{E}_{\pi^*}\left[ R(\mathbf{s}_T) \,\mid\, \mathbf{s}_t, \mathbf{a}_t \right].
\end{equation}

This formulation shows that the optimal policy is shaped jointly by the prior policy
\(\pi_{\text{old}}\) and the expected future return captured by the action-value function \(Q^*\).
Moreover, when the environment transition is deterministic~\cite{zhou2025q}, the future return is fully determined
by \(\pi_{\text{old}}\), and the optimal \(Q\)-function can be written as:
\begin{equation}
\label{easy_optimal_Q}
    Q^*(\mathbf{s}_t, \mathbf{a}_t)
    = \eta \ln \mathbb{E}_{\pi_{\text{old}}}
        \left[\exp\left(\frac{1}{\eta} R(\mathbf{s}_T)\right) \,\mid\, \mathbf{s}_t, \mathbf{a}_t\right].
\end{equation}
\subsection{Motivation}
% 这里需要在增加因为同一个scale的独立性导致optimal solution有特定结构的分析，这两件事情是并列的challenge
The key challenge in optimizing VAR sequences using RL methods lies in the asynchronous policy conflicts that arise across different timesteps. As shown in Figure~\ref{fig:motivation}, unlike diffusion models and autoregressive models, the number of query tokens fluctuates dramatically during the generation process, leading to substantial variation in task similarity across timesteps. This inconsistency is especially problematic in RL settings, where the available training data is significantly more limited compared to large-scale pretraining, thereby exacerbating policy instability and slowing convergence. In addition, because multiple tokens are generated in parallel within the same resolution scale, the optimal policy structure may deviate from that of typical sequential RL formulations, requiring further analysis.

% \subsection{Overview}
\noindent
To address the challenges mentioned above, we propose three techniques to stabilize and accelerate the RL optimization of VAR:
\begin{itemize}
    \item We introduce \textbf{Value-as-Middle-Return (VMR)} to decompose the full-sequence RL objective into a two-stage optimization problem, effectively \textit{mitigating multi-step conflict} during training.
    \item We apply \textbf{Per-Action Normalization Weighting (PANW)} to balance the contributions of different timesteps, ensuring stable learning dynamics even under varying task similarities.
    \item We implement \textbf{Mask Propagation (MP)} to prioritize updates on the most relevant tokens and to reduce imbalance across scales.
\end{itemize}
\subsection{Value as Middle Return}
\begin{figure}[t]
  \centering
  \includegraphics[width=.6\linewidth]{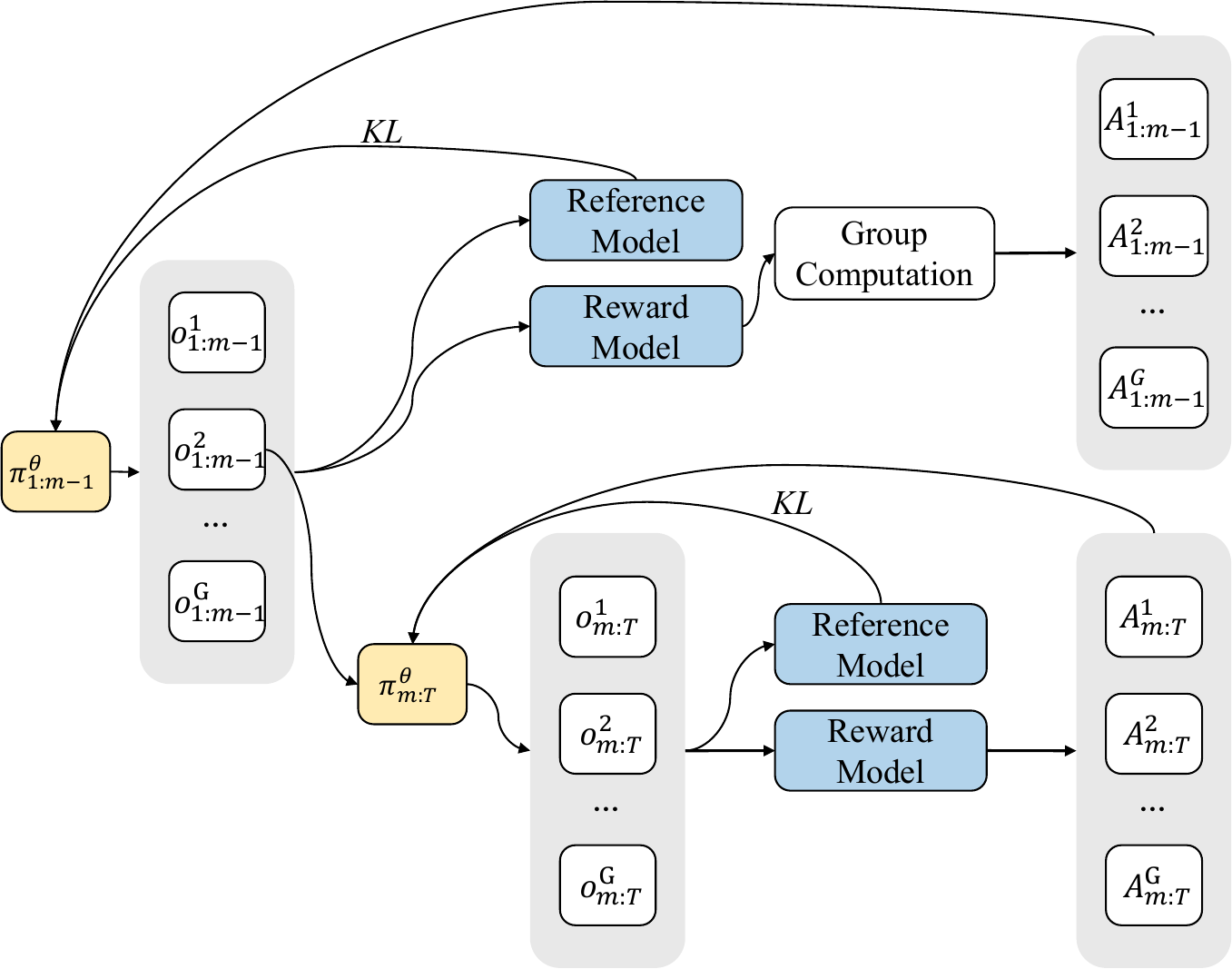}

   \caption{The pipeline of our two-stage GRPO is as follows: we use Monte Carlo estimation to compute the reward for \(\pi^\theta_{1:m-1}\), and then apply GRPO to the two states separately.}
   \label{fig:two stage grpo}
\end{figure}
\subsubsection{Formulation}
The core idea of VMR is to decompose the full-horizon, KL-regularized RL problem into two coupled subproblems split at a middle timestep \(m\). This alleviates cross-horizon action conflicts and yields denser feedback. Reusing the notation of \eqref{general optimal solution}, we define the middle-step soft value as follows:
\begin{definition}[Middle-step soft value]
\label{def:middle_soft_value}
For any \(m\in\{1,\dots,T\}\),
\begin{equation}
\label{eq:vmr_middle_value}
V_m^\ast(\mathbf{s}_m)
\;=\;
\eta \log
\mathbb{E}_{\pi_{\mathrm{old}}}
\!\left[
\exp\!\Big(\tfrac{1}{\eta} R(\mathbf{s}_T)\Big)
\,\middle|\, \mathbf{s}_m
\right].
\end{equation}
\end{definition}

\noindent
VMR introduces a middle return \(V^*_m(\mathbf{s}_m)\) to partition the generation into prefix and suffix subtasks. Each is trained independently with a local KL penalty:
\begin{equation}
\label{eq:vmr_two_stage_objectives}
\begin{aligned}
\text{Suffix}:&
\max_{\pi_{m:T-1}}
\mathbb{E}\!\big[R(\mathbf{s}_T)\mid \mathbf{s}_m\big]
-\eta\,\mathrm{KL}\!\big(\pi_{m:T-1}\,\|\,\pi_{\mathrm{old},m:T-1}\big),\\[-1pt]
\text{Prefix}:&
\max_{\pi_{1:m-1}}
\mathbb{E}\!\big[V_m^\ast(\mathbf{s}_m)\big]
-\eta\,\mathrm{KL}\!\big(\pi_{1:m-1}\,\|\,\pi_{\mathrm{old},1:m-1}\big).
\end{aligned}
\end{equation}

\noindent
This stage-wise decomposition stabilizes training by keeping per-stage token lengths comparable, while preserving the global objective value. The detailed algorithm is shown in Figure~\ref{fig:two stage grpo}.

\subsubsection{Analysis}
We now examine two central questions:

\begin{itemize}
    \item What is the optimal solution for the VAR family under KL-regularized RL?
    \item Does introducing the middle return \(V^*(\mathbf{s}_m)\) alter the optimal policy?
\end{itemize}

\textbf{VAR Family Definition.}
We first specify the policy class of interest as the following definition:

\begin{definition}[VAR family \(\mathcal{M}_{\pi}\)]
\label{def:var_family_inline}
A policy \(\pi^\theta\) belongs to \(\mathcal{M}_{\pi}\) if, at each timestep \(t\), the action grid
\(\mathbf{a}_t^\theta = \mathbf{r}_{t+1}\in\mathbb{R}^{h_{t+1}\times w_{t+1}}\)
factorizes across spatial sites given the state \(\mathbf{s}_t\):
\begin{equation}
\pi^\theta(\mathbf{a}_t\mid\mathbf{s}_t)
=\prod_{(i,j)} \pi^\theta_{t,(i,j)}(\mathbf{a}_{t,(i,j)}\mid\mathbf{s}_t).
\end{equation}
This mirrors the implementation of VAR, where tokens on the same resolution are generated in parallel. And this does not mean that different blocks of the generated image are independent, because the state \(\mathbf{s}_t\) contains all the previously generated tokens, which can provide global context.
\end{definition}

\textbf{Q1: Optimality within \(\mathcal{M}_{\pi}\).}
The globally optimal KL-regularized policy \(\pi^\ast\) from \eqref{general optimal solution} need not belong to \(\mathcal{M}_{\pi}\). Nevertheless, within this restricted family, there exists a unique constrained optimum \(\pi^\dagger\):

\begin{definition}[Constrained optimal policy]
\label{def:var_constrained}
The optimal VAR policy is
\begin{equation}
  \pi^\dagger
=\arg\max_{\pi\in\mathcal{M}_{\pi}}J(\pi), 
J(\pi)
=\mathbb{E}_{\pi}[R(s_T)]
-\eta\,\mathrm{KL}(\pi\,\|\,\pi_{\mathrm{old}}).  
\end{equation}
\end{definition}

\noindent
The $\pi^\dagger$ satisfies the following property:

\begin{restatable}[Reverse-KL characterization of \(\pi^\dagger\)]{theorem}{thmrklproj}
\label{thm:reverse_kl_projection}
At each state \(\mathbf{s}_t\), the constrained optimum satisfies
\begin{equation}
\pi^\dagger(\cdot\mid \mathbf{s}_t)
=
\arg\min_{\pi\in\mathcal{M}_{\pi}(\mathbf{s}_t)}
\mathrm{KL}\!\big(\pi(\cdot\mid \mathbf{s}_t)\,\|\,\pi^\ast(\cdot\mid \mathbf{s}_t)\big),
\end{equation}
where \(\pi^\ast\) is the global soft-optimal policy.
\end{restatable}

\noindent
Thus, the best policy within the VAR family is obtained by reverse-KL projecting the globally optimal, potentially non-factorized policy \(\pi^\ast\) onto  \(\mathcal{M}_{\pi}\).

\textbf{Q2: Effect of introducing VMR.}
We now analyze whether VMR changes the optimal solution. We have the following theorem:

\begin{restatable}[Two-stage invariance]{theorem}{thmtsinvar}
\label{thm:two_stage_invariance_var}
Let \(V_m^\ast(s_m)\) be defined as in \eqref{eq:vmr_middle_value}.  
Within \(\mathcal{M}_{\pi}\), solving the suffix problem in \eqref{eq:vmr_two_stage_objectives} yields \(\pi_{m:T-1}^\dagger\); optimizing the prefix problem using \(V_m^\ast\) as its sole reward gives \(\pi_{1:m-1}^\dagger\).  
Then, the concatenation
\(\pi^\dagger=\pi_{1:m-1}^\dagger\!\oplus\!\pi_{m:T-1}^\dagger\)
uniquely maximizes the full-horizon objective \(J(\pi)\).  
Within \(\mathcal{M}_{\pi}\), replacing each subpolicy by its per-state reverse-KL projection onto the factorized family yields the constrained optimum \(\pi^\dagger\).
\end{restatable}
\noindent
The theorem implies that the intermediate value \(V_m^\ast\) encapsulates all downstream contributions from step \(m\) onward. When used as the sole reward for the prefix, it ensures that prefix and suffix optimization remain consistent with the original full-horizon solution. This two-stage structure preserves policy invariance, avoids credit-assignment interference, and stabilizes training across varying token lengths.

\noindent
VMR provides a principled prefix–suffix decomposition of KL-regularized RL that maintains optimality within both the unconstrained and VAR-constrained settings, while reducing multi-step conflicts and improving learning stability in coarse-to-fine generation.

\subsection{Per-Action Normalization Weighting}

As shown in Figure~\ref{fig:motivation}, the task similarity across different timesteps in VAR generation can vary significantly due to the fluctuating number of query tokens. To address this issue, we propose Per-Action Normalization Weighting, which normalize the contributions of each action based on the number of query tokens at that timestep. Specifically, for each timestep \(t\), we compute a normalization weight \(k_t\) defined as:
\begin{equation}
\label{equ: coeff}
k_t = \frac{1}{(h_t \times w_t)^\alpha}
\end{equation}
where \(h_t\) and \(w_t\) are the height and width of the token grid at timestep \(t\) and \(\alpha\) is a decay exponent. During training, we weight the loss associated with each action \(\mathbf{a}_t\) by its corresponding normalization weight \(k_t\). This approach ensures that timesteps with a larger number of query tokens do not disproportionately influence the learning process, thereby promoting a more balanced optimization across all timesteps.

\subsection{Mask Propagation}
% 缺少一个示意图
\begin{figure}[t]
  \centering
  \includegraphics[width=.45\linewidth]{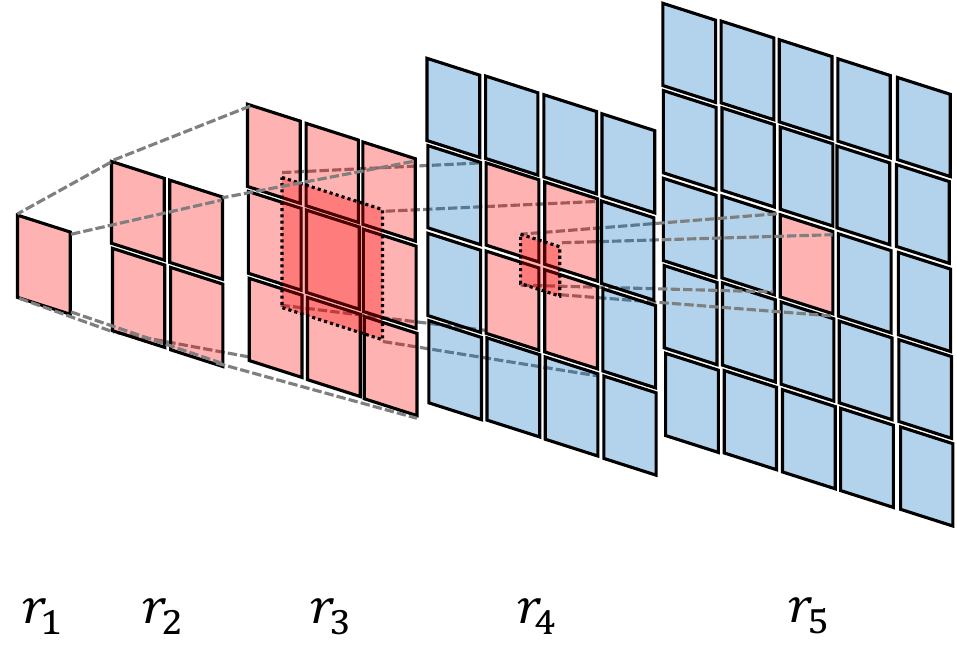}

   \caption{Our token masks are propagated in reverse through the model’s multi-scale hierarchy, moving from finer to coarser feature scales.}
   \label{fig:mask prop}
\end{figure}

To precisely identify which tokens influence the final return, we introduce Mask Propagation. As shown in Figure~\ref{fig:mask prop}, we first construct an initial mask from the output components that directly determine the reward (e.g., predicted bounding boxes). We then propagate this mask backward through the model’s multi-scale hierarchy, moving from finer to coarser feature scales. This process directs updates toward the most relevant tokens while simultaneously improving cross-scale balance.
\section{Experiments}
\label{experiments}
\subsection{Experimental Setup}
We use NextFlow \cite{qu2025tokenflow} as our base model. Images are generated at a resolution of 1024×1024. 

% 采样策略

% Training setup (LaTeX snippet)
\textbf{Training Setup.}
We adopt an on-policy training strategy with group size $=16$ (candidates per prompt) and batch size $=16$ (prompts per update).
The policy is initialized with a learning rate of $10^{-6}$.
Prompts are sampled from our in-house pool and are strictly disjoint from the training and testing splits of all evaluation tasks.
We train for up to $1{,}200$ updates, i.e., up to $19{,}200$ unique prompts per task.
Optimization uses \textsc{AdamW} with learning rate $10^{-5}$, default $\beta_1 = 0.9$, $\beta_2 = 0.95$, and weight decay $= 0.05$.
We alternate optimization at different scales following Eq.~\eqref{eq:vmr_two_stage_objectives}: for every three prefix GRPO updates (optimizing $\pi_{1:m-1}$), we perform one suffix GRPO update (optimizing $\pi_{m:T-1}$). In contrast to the sampling configuration, classifier-free guidance (CFG)~\cite{ho2022classifier} is not used during training.

\textbf{Sampling Setup.}
We adopt the original sampling strategy of VAR \cite{var} for both rollout and evaluation phases. The sampling parameters are fixed at CFG=5, top-k=2 and top-p=0.9 for all token sampling operations.

% VMR Reward Construction (sample-based version)
\textbf{VMR Reward Construction.}
Instead of fitting a step-wise critic as in PPO~\cite{schulman2017proximal}, we estimate the middle-stage value directly from
on-policy terminal rewards. For a given state $\mathbf{s}_m$, we sample $K$ on-policy rollouts
$\{\tau_k\}_{k=1}^K \sim \pi_{\mathrm{\theta}}(\cdot \mid \mathbf{s}_m)$ up to terminal step $T$, obtain
terminal rewards $\{R^{(k)}(\mathbf{s}_T)\}_{k=1}^K$, and define the VMR as a risk-sensitive estimator:
\begin{equation}
\label{eq:vmr_middle_value_mc}
\widehat{V}_m(\mathbf{s}_m)
\;=\;
\eta \,\log \!\left(
\frac{1}{K} \sum_{k=1}^{K}
\exp\!\Big(\tfrac{1}{\eta} \, R^{(k)}(\mathbf{s}_T)\Big)
\right),
\end{equation}
where $\eta>0$ is a temperature parameter and $K$ is the number of on-policy samples.
In practice, we set $\eta = 1$ and $K = 2$, which we find sufficient and stable (Figure~\ref{fig:motivation_2}).
We primarily validate our method on the text rendering task with an extensive ablation study, and further evaluate on HPSv3 to demonstrate robustness and generalization.

\textbf{Selection of \(m\).}
We observe that the primary gains stem from applying RL to the prefix segment, making the choice of $m$ crucial. As shown in Figure~\ref{fig:motivation_2}, a sweet spot emerges at $m\in\{m_{128},m_{256}\}$, where these values represent the steps corresponding to resolutions $128\times128$ and $256\times256$ (see Table~\ref{tab:patch_num_list}). Further ablations (Subsection~\ref{subsec: ablation}) justify our final choice of $m=m_{256}$.

\subsection{Text Rendering}
\begin{figure}[htbp]
  \centering
  \includegraphics[width=.7\linewidth]{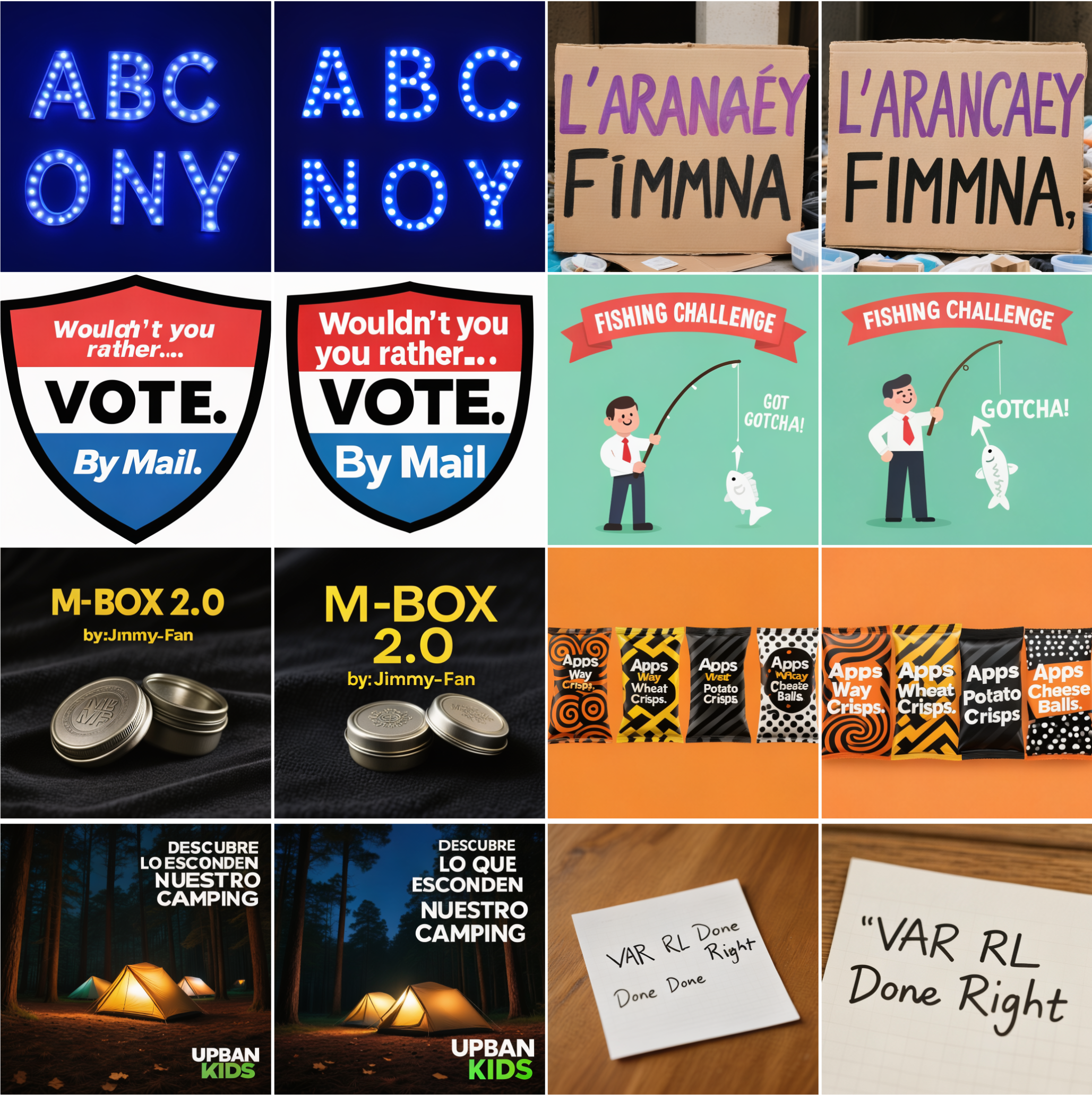}

   \caption{Visual Samples of text rendering task before (left) and after (right) RL optimization. The text required for each pairs: (1) Six illuminated letters ('A', 'B', 'C', \textcolor{red}{'N', 'O'}, 'Y') (2) "Woul\textcolor{red}{dn'}t you rather... VOTE. By Mail." (3) "M-BOX 2.0 by: J\textcolor{red}{im}my-Fan" (4) "DESCUBRE LO \textcolor{red}{QUE} ESCONDEN NUESTRO CAMPING" and "UPBAN KIDS" (5) "L'ARA\textcolor{red}{NCA}EY FIMMNA," (6) "FISHING CHALLENGE" and \textcolor{red}{"GOTCHA!"} (7) "Apps Way \textcolor{red}{Crisps}.", "Apps \textcolor{red}{Wheat} Crisps.", "Apps \textcolor{red}{Potato} Crisps." and "Apps \textcolor{red}{Cheese Balls}." (8) "VAR RL \textcolor{red}{Done Right}". Full captions are provided in the Appendix. The RL-refined outputs demonstrate improvements in correcting character misordering, erroneous glyphs, missing or extraneous characters. Better zoom in for details.}
   \label{fig:ocr_compare}
\end{figure}
% Single-column equations with labels and clear notation
\subsubsection{Reward Design}

\textbf{OCR Backbone.}
We use PaddleOCRv5~\cite{cui2025paddleocr30technicalreport,cui2025paddleocrvlboostingmultilingualdocument} to recognize text from rendered images. Let the lowercased ground-truth word sequence be $\mathcal{G}=(g_k)_{k=1}^{N}$, the predicted word sequence be $\mathcal{P}=(p_i)_{i=1}^{M}$, and the corresponding OCR confidences be $\mathcal{S}=(s_i)_{i=1}^{M}$ with $s_i\in[0,1]$. We measure string similarity by a normalized Levenshtein score:
\begin{equation}
\operatorname{LD}(x,y)=1-\frac{\operatorname{EditDist}(x,y)}{\max\{\,|x|,|y|\,\}+\varepsilon},
\label{eq:ld}
\end{equation}
where $|x|$ is the character length of $x$, $\operatorname{EditDist}(\cdot,\cdot)$ is the Levenshtein distance, and $\varepsilon>0$ avoids division by zero.

\noindent
\textbf{Completeness (confidence-aware).} For each $g\!\in\!\mathcal{G}$, if $g$ appears in $\mathcal{P}$, we count the minimum confidence among identical predictions to discourage duplicate inflation:
\begin{equation}
\operatorname{Comp}=\frac{1}{N}\sum_{g\in\mathcal{G}}\Big(\mathbf{1}\{g\in\mathcal{P}\}\cdot \min_{i:\,p_i=g} s_i\Big),
\label{eq:comp}
\end{equation}
where $\mathbf{1}\{\cdot\}$ is the indicator function.

\noindent
\textbf{Length mismatch penalty.} To penalize over/under-generation, we compare concatenated strings via a multiset distance:
\begin{equation}
\operatorname{Pen}=\lambda\cdot
\frac{\operatorname{BagDist}\!\left(\sum_{i=1}^{M} p_i,\ \sum_{k=1}^{N} g_k\right)}
{\left|\sum_{i=1}^{M} p_i\right|+\left|\sum_{k=1}^{N} g_k\right|},
\quad \lambda=0.6,
\label{eq:pen}
\end{equation}
where $\operatorname{BagDist}(\cdot,\cdot)$ measures character-level multiset discrepancy and the denominator normalizes by total character count.

\noindent
\textbf{Similarity (confidence-weighted).} Let $i^\star(g)=\arg\max_{i}\operatorname{LD}(g,p_i)$ be the best match index for $g$; we can weight by the matched confidence:
\begin{equation}
\operatorname{Sim}=\frac{1}{N}\sum_{g\in\mathcal{G}}\operatorname{LD}\!\big(g,p_{i^\star(g)}\big)\, s_{i^\star(g)}
\label{eq:rewardc}
\end{equation}

\noindent
\textbf{Total reward.} The final OCR reward combines the above:
\begin{equation}
\operatorname{Reward}=\operatorname{Comp}+\operatorname{Sim}-\operatorname{Pen}.
\label{eq:reward}
\end{equation}

\subsubsection{Experimental Analysis}
On CVTG-2K~\cite{du2025textcrafter} (Table~\ref{tab:cvtg2k_results_condensed_resized}), our RL-based method (\ALG{}) markedly outperforms NextFlow across all metrics. Specifically, \ALG{} improves Word Accuracy from 0.5536 to 0.7841 (+0.2305 absolute, +41.6\% relative) and NED from 0.7816 to 0.9081 (+0.1265 absolute, +16.2\% relative), while also achieving a higher CLIPScore (0.8224 vs.\ 0.8068). These results demonstrate that \ALG{}, driven by confidence- and similarity-aware OCR rewards with a length penalty, significantly improves word-level accuracy and character-level fidelity while preserving semantic alignment. 

\noindent
Overall, the consistent improvements over NextFlow and the leading results among diffusion models validate the effectiveness of our two-stage RL scheme: optimizing early-token (prefix) decisions guided by confidence/similarity-aware OCR rewards with a length penalty yields robust gains in both text fidelity and visual quality across broad content categories.
% 你可以直接写在这吧 ok
\subsubsection{Visual Results}
Figure~\ref{fig:ocr_compare} shows visual results before (left in each pair) and after (right) RL optimization on text rendering. RL-refined outputs correct character order, fix erroneous glyphs, and reduce missing or extraneous characters across diverse layouts and styles. Full prompts are provided in the Appendix.

{
\setlength{\tabcolsep}{22pt} 
\begin{table}[t]
\centering
\caption{Quantitative results on the CVTG-2K dataset. Bold denotes the best performance, underline denotes the second-best for each metric. Seedream 3.0 and GPT Image 1 are highlighted as proprietary/closed-source systems.}
\label{tab:cvtg2k_results_condensed_resized}
\resizebox{\columnwidth}{!}{%
\begin{tabular}{@{}lcccc@{}}
\toprule
\textbf{Model} & \textbf{\#Params} & \textbf{Word Accuracy}$\uparrow$ & \textbf{NED}$\uparrow$ & \textbf{CLIPScore}$\uparrow$ \\
\midrule
FLUX.1 dev \cite{flux} & 12B & 0.4965 & 0.6879 & 0.7401 \\
SD3.5 Large \cite{SD3} & 8B & 0.6548 & 0.8470 & 0.7797 \\
AnyText \cite{tuo2023anytext} & 1.4B & 0.1804 & 0.4675 & 0.7432 \\
TextDiffuser-2 \cite{chen2024textdiffuser} & 8B & 0.2326 & 0.4353 & 0.6765 \\
RAG-Diffusion \cite{li2025ragdiffusion} & 12B & 0.2648 & 0.4498 & 0.6688 \\
3DIS \cite{zhou20243dis} & 12B & 0.3813 & 0.6505 & 0.7767 \\
TextCrafter (FLUX) \cite{du2025textcrafter} & 12B & 0.7370 & 0.8679 & 0.7868 \\
TextCrafter (SD3.5) \cite{du2025textcrafter} & 8B & \underline{0.7600} & \underline{0.9038} & 0.8023 \\
\midrule
% \rowcolor{lightgray}
Seedream 3.0 \cite{gao2025seedream} & - & 0.5924 & 0.8537 & 0.7821 \\
% \rowcolor{lightgray}
GPT Image 1 [High] \cite{gptimage} & - & \textbf{0.8569} & \textbf{0.9478} & \underline{0.7982} \\
Qwen-Image \cite{wu2025qwen} & 20B & \underline{0.8288} & \underline{0.9116} & \textbf{0.8017} \\
\midrule
NextFlow & 7B & 0.5536 & 0.7816 & \underline{0.8068} \\
\ALG{}  & 7B & \textbf{0.7841} & \textbf{0.9081} & \textbf{0.8224} \\
\bottomrule
\end{tabular}%
}
\end{table}
}

\subsection{Human Preference Score}
\begin{figure}[htbp]
  \centering
  \includegraphics[width=\linewidth]{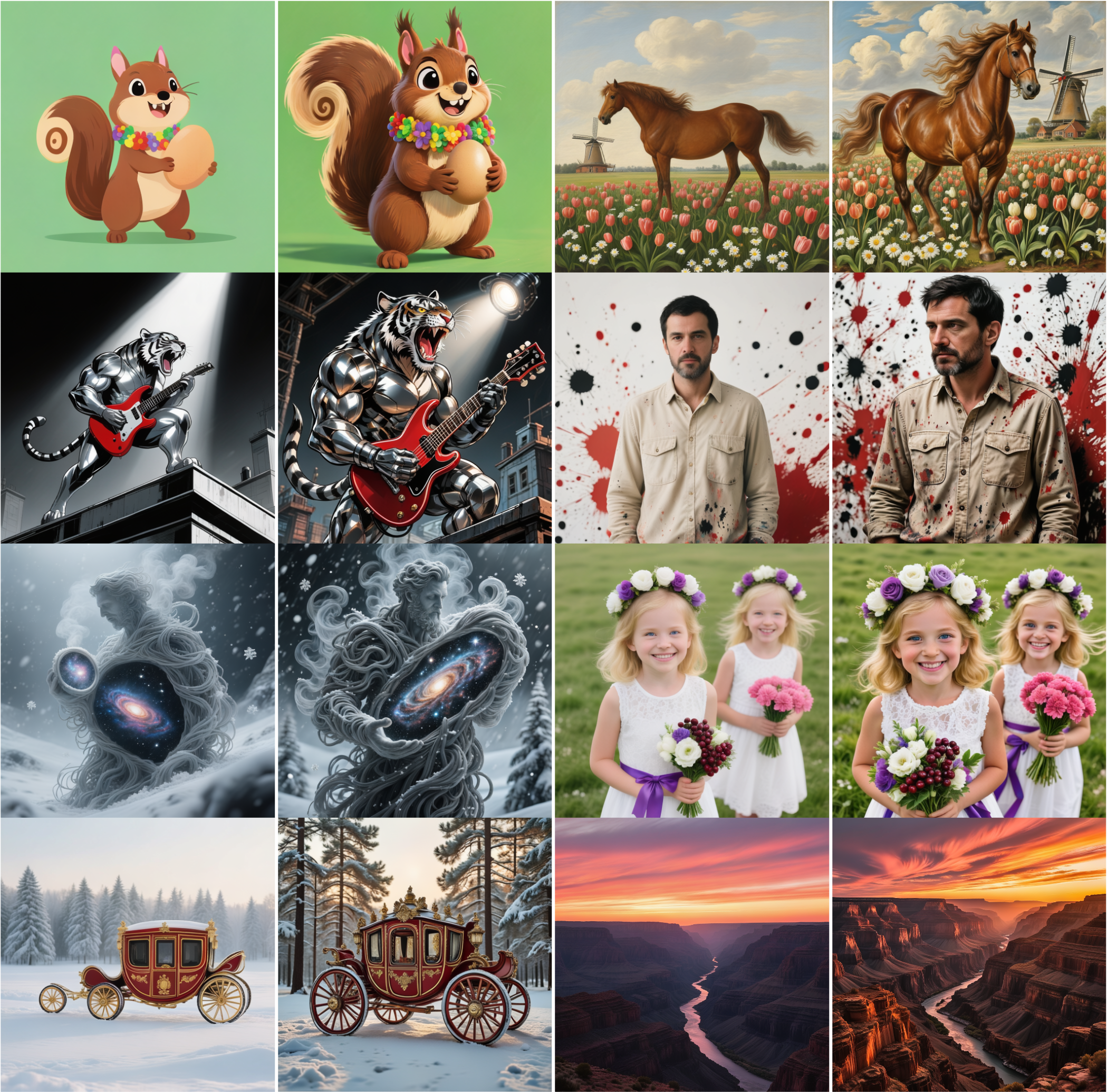}

   \caption{Visual Samples of HPS refine task before (left) and after (right) RL optimization. Full captions are provided in the Appendix. Better zoom in for details.}
   \label{fig:hps_visual}
\end{figure}
\subsubsection{Reward Design}
We adopt the best hyperparameters from the Text Rendering task and use \textbf{HPSv3}~\cite{ma2025hpsv3} \textit{as the direct reward function} during RL optimization. To handle the large model efficiently, we deploy a self-hosted HPS service for low-latency reward evaluation and high-throughput training.

\subsubsection{Experimental Analysis}
Table~\ref{tab:model_comparison_reversed} shows that our method (\ALG{}) delivers substantial gains over the starting point NextFlow across all categories. On the overall metric \textbf{All}, \ALG{} improves from 8.43 to \textbf{10.64} (+2.21 absolute), and consistently outperforms NextFlow in every sub-category, with notable margins on \textbf{Animals} (10.79 vs.\ 8.34), \textbf{Natural Scenery} (10.43 vs.\ 8.22), \textbf{Plants} (10.57 vs.\ 8.19), and \textbf{Food} (11.18 vs.\ 8.93). These results indicate that our RL-based prefix optimization and OCR-guided reward design not only enhance word rendering fidelity but also generalize across diverse visual domains (e.g., transportation, products, and science).

Among diffusion-centric models, \ALG{} attains state-of-the-art performance on multiple key columns: it ranks \textbf{1st} on \textbf{All} (10.64), \textbf{Architecture} (11.16), \textbf{Animals} (10.79), \textbf{Natural Scenery} (10.43), \textbf{Plants} (10.57), \textbf{Food} (11.18), and \textbf{Others} (10.66), while remaining competitive on \textbf{Characters} (11.72, second only to Kolors at 11.79) and \textbf{Design} (9.77, second only to Kolors at 9.87). Compared to strong baselines such as Flux-dev, Kolors, and Playground-v2.5, \ALG{} matches or surpasses their best scores in most categories, establishing a new SOTA within the diffusion-based family.

\subsubsection{Visual Results}
We showcase visual results on the HPS refine task before and after RL optimization in Figure~\ref{fig:hps_visual}. Across diverse prompts and styles, including cartoons, photorealistic portraits, astrophotography, landscapes, and object-centric scenes, our RL-enhanced model produces images with sharper details, cleaner structure, and stronger prompt adherence. Full captions are provided in the Appendix~\ref{sec.prompts}.

\begin{table*}[htbp]
\centering
\caption{Comparison of different models across various categories. The best results are in \textbf{bold}, and the second-best are \underline{underlined}.}
\label{tab:model_comparison_reversed}
\resizebox{\textwidth}{!}{%
\begin{tabular}{@{}l|rrrrrrrrrrrrr@{}}
\toprule
\textbf{Models} & \textbf{All} & \textbf{Characters} & \textbf{Arts} & \textbf{Design} & \textbf{Architecture} & \textbf{Animals} & \textbf{Natural Scenery} & \textbf{Transportation} & \textbf{Products} & \textbf{Plants} & \textbf{Food} & \textbf{Science} & \textbf{Others} \\
\midrule
Stable Diffusion v2.0 \cite{Rombach_2022_CVPR} & -0.24 & -0.34 & -0.56 & -1.35 & -0.24 & -0.54 & -0.32 & 1.00 & 1.11 & -0.01 & -0.38 & -0.38 & -0.84 \\
Stable Diffusion 3 \cite{SD3} & 5.31 & 6.70 & 5.98 & 5.15 & 5.25 & 4.09 & 5.24 & 4.25 & 5.71 & 5.84 & 6.01 & 5.71 & 4.58 \\
Hunyuan \cite{li2024hunyuan} & 8.19 & 7.96 & 8.11 & 8.28 & 8.71 & 7.24 & 7.86 & 8.33 & 8.55 & 8.28 & 8.31 & 8.48 & 8.20 \\
Stable Diffusion XL \cite{podell2023sdxl} & 8.20 & 8.67 & 7.63 & 7.53 & 8.57 & 8.18 & 7.76 & 8.65 & 8.85 & 8.32 & 8.43 & 8.78 & 7.29 \\
Gemini 2.0 Flash \cite{comanici2025gemini} & 9.21 & 9.98 & 8.44 & 7.64 & 10.11 & 9.42 & 9.01 & 9.74 & 9.64 & 9.55 & 10.16 & 7.61 & 9.23 \\
PixArt-$\Sigma$ \cite{chen2024pixart} & 9.37 & 10.08 & 9.07 & 8.41 & 9.83 & 8.86 & 8.87 & 9.44 & 9.57 & 9.52 & 9.73 & 10.35 & 8.58 \\
CogView4 \cite{zheng2024cogview3} & 9.61 & 10.72 & 9.86 & 9.33 & 9.88 & 9.16 & 9.45 & 9.69 & 9.86 & 9.45 & 9.49 & 10.16 & 8.97 \\
Infinity \cite{han2025infinity} & 10.26 & 11.17 & 9.95 & 9.43 & 10.36 & 9.27 & \underline{10.11} & 10.36 & 10.59 & 10.08 & 10.30 & 10.59 & \underline{9.62} \\
Playground-v2.5 \cite{li2024playground} & 10.27 & 11.07 & 9.84 & 9.64 & 10.45 & 10.38 & 9.94 & 10.51 & 10.62 & 10.15 & 10.62 & 10.84 & 9.39 \\
Flux-dev \cite{flux} & 10.43 & 11.70 & \underline{10.32} & 9.39 & \underline{10.93} & 10.38 & 10.01 & \textbf{10.84} & \textbf{11.24} & 10.21 & 10.38 & \textbf{11.24} & 9.16 \\
Kolors \cite{kolors} & \underline{10.55} & \textbf{11.79} & \textbf{10.47} & \textbf{9.87} & 10.82 & \underline{10.60} & 9.89 & 10.68 & \underline{10.93} & \underline{10.50} & \underline{10.63} & \underline{11.06} & 9.51 \\
\midrule
NextFlow  & 8.43 & 9.27 & 8.00 & 7.51 & 8.98 & 8.34 & 8.22 & 8.68 & 8.70 & 8.19 & 8.93 & 7.75 & 8.57 \\
\ALG{}  & \textbf{10.64} & \underline{11.72} & 10.26 & \underline{9.77} & \textbf{11.16} & \textbf{10.79} & \textbf{10.43} & \underline{10.77} & 10.73 & \textbf{10.57} & \textbf{11.18} & 9.60 & \textbf{10.66} \\
\bottomrule
\end{tabular}%
}
\end{table*}

\subsection{Ablation}
\label{subsec: ablation}

% We doing ablation on the following design choices and hyperparameter:
We conduct ablations over key design choices and hyperparameter, summarizing findings below.

% \begin{table}[t]
% \centering
% \caption{Ablation on the choice of $m$, with $K=2$ at step 300.}
% \label{tab:ablation_m_selection}
% \resizebox{.5\columnwidth}{!}{%
% \begin{tabular}{@{}lccc@{}}
% \toprule
% \textbf{$m$} & \textbf{Word Accuracy}$\uparrow$ & \textbf{NED}$\uparrow$ & \textbf{CLIPScore}$\uparrow$ \\
% \midrule
% NextFlow & 0.5536 & 0.7816 & 0.8068 \\
% \hdashline
% 1024 (vinilla GRPO)   & 0.5741 & 0.7871 & 0.8038 \\
% 512   & 0.6351 & 0.8236 & 0.8057 \\
% 256  & \underline{0.6565} & \underline{0.8429} & \underline{0.8133} \\
% 128  & \textbf{0.6677} & \textbf{0.8501} & \textbf{0.8142} \\
% \bottomrule
% \end{tabular}%
% }
% \end{table}

\begin{table}[t] % 使用一个 table 环境包裹所有内容
    \centering
    \caption{Ablation studies on the hyperparameters $m$ and $\alpha$.} % 创建一个总标题
    \label{tab:main_ablation}

    % --- 第一个子表格 ---
    \begin{subtable}[t]{0.48\textwidth}
        \centering
        \caption{Ablation on the choice of $m$, with $K=2$ at step 300.}
        \label{tab:ablation_m_selection}
        % 使用 \resizebox{\columnwidth}{!}{...} 让表格自动适应 subtable 的宽度
        \resizebox{\columnwidth}{!}{%
            \begin{tabular}{@{}lccc@{}}
            \toprule
            \textbf{$m$} & \textbf{Word Accuracy}$\uparrow$ & \textbf{NED}$\uparrow$ & \textbf{CLIPScore}$\uparrow$ \\
            \midrule
            NextFlow & 0.5536 & 0.7816 & 0.8068 \\
            \hdashline
            1024   & 0.5741 & 0.7871 & 0.8038 \\ % 修正了 vanilla 的拼写
            512   & 0.6351 & 0.8236 & 0.8057 \\
            256  & \underline{0.6565} & \underline{0.8429} & \underline{0.8133} \\
            128  & \textbf{0.6677} & \textbf{0.8501} & \textbf{0.8142} \\
            \bottomrule
            \end{tabular}%
        }
    \end{subtable}
    \hfill % 在两个子表格之间添加弹性空白，使它们左右分开
    % --- 第二个子表格 ---
    \begin{subtable}[t]{0.48\textwidth}
        \centering
        \caption{Ablation on decay exponent $\alpha$, measured at step 400.}
        \label{tab:ablation_reweight}
        \resizebox{\columnwidth}{!}{%
            \begin{tabular}{@{}lccc@{}}
            \toprule
            \textbf{$\alpha$} & \textbf{Word Accuracy}$\uparrow$ & \textbf{NED}$\uparrow$ & \textbf{CLIPScore}$\uparrow$ \\
            \midrule
            NextFlow & 0.5536 & 0.7816 & 0.8068 \\
            \hdashline
            1.2 & 0.6602 & 0.8459 & 0.8144 \\
            1.0 & 0.6855 & 0.8601 & 0.8188 \\
            0.8 & \underline{0.7051} & \underline{0.8673} & \textbf{0.8172} \\
            0.6 & \textbf{0.7136} & \textbf{0.8709} & \underline{0.8165} \\
            \bottomrule
            \end{tabular}%
        }
    \end{subtable}
\end{table}

\begin{table}[htbp]
    \centering
    \caption{Ablation studies on Mask Propagation (MP) and alternating training schemes.} % 一个概括性的总标题
    \label{tab:main_ablation_2}
    % --- 第一个子表格 (Mask Propagation) ---
    \begin{subtable}[t]{0.48\textwidth}
        \centering
        \caption{Effect of MP, measured at step 400.}
        \label{tab:ablation_mask}
        \resizebox{\columnwidth}{!}{%
            \begin{tabular}{@{}lccc@{}}
            \toprule
            \textbf{Model} & \textbf{Word Accuracy}$\uparrow$ & \textbf{NED}$\uparrow$ & \textbf{CLIPScore}$\uparrow$ \\
            \midrule
            NextFlow & 0.5536 & 0.7816 & 0.8068 \\
            \hdashline
            \ALG{} w/o MP & 0.6855 & 0.8601 & \textbf{0.8188} \\
            \ALG{} w/ MP & \textbf{0.7071} & \textbf{0.8699} & 0.8184 \\
            \bottomrule
            \end{tabular}%
        }
    \end{subtable}
    \hfill % 弹性空白，将两个子表格推向两侧
    % --- 第二个子表格 (Alternating Training) ---
    \begin{subtable}[t]{0.48\textwidth}
        \centering
        \caption{Ablation on alternating training schemes.} % 标题可以稍微简化
        \label{tab:ablation_switch}
        \resizebox{\columnwidth}{!}{%
            \begin{tabular}{@{}lccc@{}}
            \toprule
            \textbf{Scheme} & \textbf{Word Accuracy}$\uparrow$ & \textbf{NED}$\uparrow$ & \textbf{CLIPScore}$\uparrow$ \\ % 将 Model 改为 Scheme 更贴切
            \midrule
            NextFlow & 0.5536 & 0.7816 & 0.8068 \\
            \hdashline
            Fine-grained & \textbf{0.6855} & \textbf{0.8601} & \textbf{0.8188} \\
            Coarse-grained & 0.6778 & 0.8564 & 0.8168 \\
            \bottomrule
            \end{tabular}%
        }
    \end{subtable}
\end{table}

\textbf{Selection of $m$ in Equation~\eqref{eq:vmr_middle_value_mc}.}
From Table~\ref{tab:ablation_m_selection}, $m{=}m_{128}$ achieves the best scores (Word Acc.\ 0.6677, NED 0.8501, CLIPScore 0.8142), while $m{=}m_{256}$ is very close (0.6565/0.8429/0.8133). Given comparable quality but lower compute to estimate the VMR and better compatibility with the mask mechanism, we select $m{=}m_{256}$ as the default. Larger $m$ (e.g., 512, 1024) underperform, indicating the benefit of placing the VMR earlier to reduce variance and improve early credit assignment.

% \begin{table}[t]
% \centering
% \caption{Ablation on decay exponent $\alpha$, measured at step 400.
% }
% \label{tab:ablation_reweight}
% \resizebox{\columnwidth}{!}{%
% \begin{tabular}{@{}lccc@{}}
% \toprule
% \textbf{$\alpha$} & \textbf{Word Accuracy}$\uparrow$ & \textbf{NED}$\uparrow$ & \textbf{CLIPScore}$\uparrow$ \\
% \midrule
% NextFlow & 0.5536 & 0.7816 & 0.8068 \\
% \hdashline
% 1.2 & 0.6602 & 0.8459 & 0.8144 \\
% 1.0 & 0.6855 & 0.8601 & 0.8188 \\
% 0.8 & \underline{0.7051} & \underline{0.8673} & \textbf{0.8172} \\
% 0.6 & \textbf{0.7136} & \textbf{0.8709} & \underline{0.8165} \\
% \bottomrule
% \end{tabular}%
% }
% \end{table}

% \begin{table}[t]
% \centering
% \caption{Evaluation results across runs.}
% \label{tab:eval_runs}
% \resizebox{\columnwidth}{!}{%
% \begin{tabular}{@{}lccc@{}}
% \toprule
% \textbf{Training Steps} & \textbf{Word Accuracy}$\uparrow$ & \textbf{NED}$\uparrow$ & \textbf{CLIPScore}$\uparrow$ \\
% \midrule
% 0 & 0.4492 & 0.7265 & 0.8244 \\
% 6000 & 0.5186 & 0.7636 & 0.8240 \\
% 16000 & 0.5500 & 0.7828 & 0.8253 \\
% 24000 & 0.5418 & 0.7767 & 0.8258 \\
% \bottomrule
% \end{tabular}%
% }
% \end{table}

\textbf{Decay exponent $\alpha$ in Equation~\eqref{equ: coeff}.}
As shown in Table~\ref{tab:ablation_reweight} and Figure~\ref{fig:ablation_plots:a}, $\alpha \in [0.6, 0.8]$ yields the strongest results, with $\alpha{=}0.6$ giving the best Word Accuracy/NED and $\alpha{=}0.8$ the top CLIPScore. This range provides a robust default, balancing gradient normalization across heterogeneous steps without over-suppressing high-resolution updates.

% \begin{table}[t]
% \centering
% \caption{Effect of Mask Propagation (MP), measured at step 400.
% }
% \label{tab:ablation_mask}
% \resizebox{\columnwidth}{!}{%
% \begin{tabular}{@{}lccc@{}}
% \toprule
% \textbf{Model} & \textbf{Word Accuracy}$\uparrow$ & \textbf{NED}$\uparrow$ & \textbf{CLIPScore}$\uparrow$ \\
% \midrule
% NextFlow & 0.5536 & 0.7816 & 0.8068 \\
% \hdashline
% \ALG{} w/o MP & 0.6855 & 0.8601 & \textbf{0.8188} \\
% \ALG{} w/ MP & \textbf{0.7071} & \textbf{0.8699} & 0.8184 \\
% \bottomrule
% \end{tabular}%
% }
% \end{table}

\textbf{Mask Propagation (MP).}
Table~\ref{tab:ablation_mask} and Figure~\ref{fig:ablation_plots:b} show that enabling MP improves text fidelity metrics (Word Acc.\ 0.7071 vs.\ 0.6855; NED 0.8699 vs.\ 0.8601), with essentially unchanged CLIPScore. This confirms MP sharpens spatiotemporal credit assignment and is beneficial at scale.
\begin{table}[htbp]
\centering
\caption{Ablation on $K$, evaluated at step 300.}
\label{tab:ablation_samples}
\resizebox{.48\columnwidth}{!}{%
\begin{tabular}{@{}lccc@{}}
\toprule
\textbf{$K$} & \textbf{Word Accuracy}$\uparrow$ & \textbf{NED}$\uparrow$ & \textbf{CLIPScore}$\uparrow$ \\
\midrule
NextFlow & 0.5536 & 0.7816 & 0.8068 \\
\hdashline
1 & 0.6575 & 0.8388 & 0.8258 \\
2 & \textbf{0.6821} & \textbf{0.8505} & \textbf{0.8287} \\
4 & 0.6720 & 0.8449 & 0.8278 \\
\bottomrule
\end{tabular}%
}
\end{table}
\begin{figure}[htbp]
  \centering
  \begin{subfigure}[t]{0.5\linewidth}
    \centering
    \includegraphics[width=\linewidth]{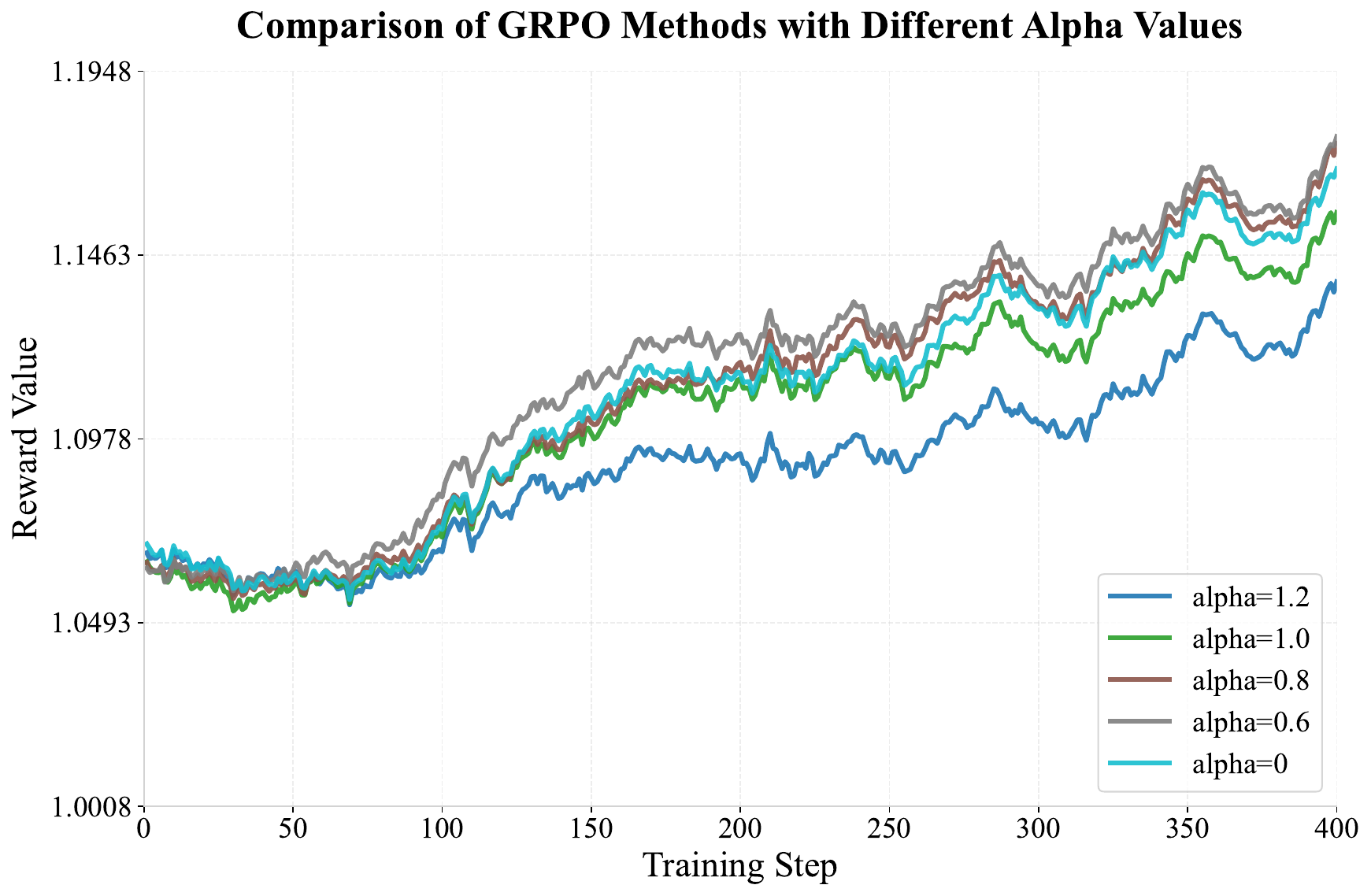}
    \caption{Decay exponent $\alpha$ comparison.}
    \label{fig:ablation_plots:a}
  \end{subfigure}\hfill
  \begin{subfigure}[t]{0.5\linewidth}
    \centering
    \includegraphics[width=\linewidth]{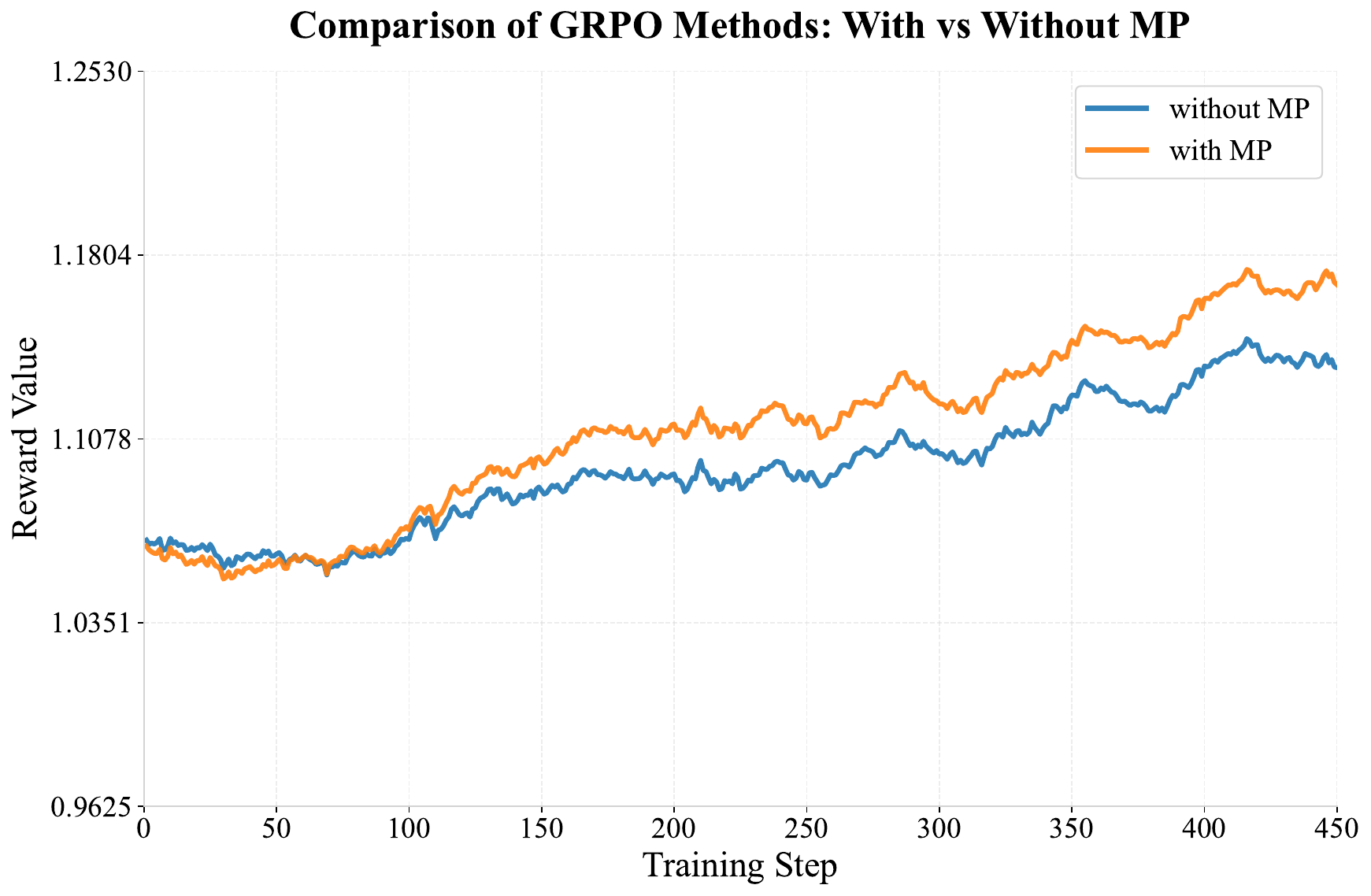}
    \caption{Effect of Mask Propagation (MP).}
    \label{fig:ablation_plots:b}
  \end{subfigure}
  \caption{Ablations on training configurations. Left: varying decay exponent $\alpha$; Right: enabling/disabling MP.}
  \label{fig:ablation_plots}
\end{figure}

% \begin{table}[t]
% \centering
% \caption{Ablation on alternating training between the low-scale (prefix tokens) and high-scale (suffix tokens) segments, comparing fine-grained versus coarse-grained alternation schemes.}
% \label{tab:ablation_switch}
% \resizebox{\columnwidth}{!}{%
% \begin{tabular}{@{}lccc@{}}
% \toprule
% \textbf{Model} & \textbf{Word Accuracy}$\uparrow$ & \textbf{NED}$\uparrow$ & \textbf{CLIPScore}$\uparrow$ \\
% \midrule
% NextFlow & 0.5536 & 0.7816 & 0.8068 \\
% \hdashline
% Fine-grained & \textbf{0.6855} & \textbf{0.8601} & \textbf{0.8188} \\
% Coarse-grained & 0.6778 & 0.8564 & 0.8168 \\
% \bottomrule
% \end{tabular}%
% }
% \end{table}

\textbf{Number of samples $K$ in Equation~\eqref{eq:vmr_middle_value_mc}.}
From Table~\ref{tab:ablation_samples}, $K{=}2$ performs best overall (0.6821/0.8505/0.8287). While $K{=}1$ underutilizes exploration, larger $K$ (e.g., $4$) shows slight degradation, suggesting compatibility issues with MP and increased variance from heterogeneous trajectories. Thus, $K{=}2$ offers the best stability–performance trade-off.

\textbf{Alternating training granularity.}
Table~\ref{tab:ablation_switch} compares fine-grained versus coarse-grained alternation between low-scale (prefix) and high-scale (suffix) segments. In the fine-grained scheme, we perform three prefix updates for every suffix update; in the coarse-grained scheme, we first run 300 prefix updates and then 100 suffix updates. Both settings are evaluated at the same total training step (400). Fine-grained alternation consistently outperforms coarse-grained (0.6855/0.8601/0.8188 vs.\ 0.6778/0.8564/0.8168), indicating that more frequent, localized updates better resolve asynchronous policy conflicts.

\section{Conclusion}
We present the first RL framework for VAR in T2I task, tackling asynchronous policy conflicts across heterogeneous scales. Our VMR decomposes the full-horizon objective into prefix/suffix stages without altering the optimal solution, while Per-Action Normalization Weighting and Mask Propagation stabilize credit assignment and concentrate updates on reward-relevant tokens. Empirically, our method delivers consistent gains over vanilla GRPO with strong improvements over original NextFlow.
% We present the first RL framework for VAR in T2I task, tackling asynchronous policy conflicts across heterogeneous scales. Our VMR decomposes the full-horizon objective into prefix/suffix stages without altering the optimal solution, while Per-Action Normalization Weighting and Mask Propagation stabilize credit assignment and concentrate updates on reward-relevant tokens. Empirically, our method delivers consistent gains over vanilla GRPO with strong improvements over original TokenFlow-T2I.

\clearpage

\bibliographystyle{plainnat}
\bibliography{main}

\clearpage
\onecolumn
\beginappendix
% \clearpage
% \setcounter{page}{1}
% \maketitlesupplementary

% \section{Rationale}
% \label{sec:rationale}
% % 
% Having the supplementary compiled together with the main paper means that:
% % 
% \begin{itemize}
% \item The supplementary can back-reference sections of the main paper, for example, we can refer to \cref{sec:intro};
% \item The main paper can forward reference sub-sections within the supplementary explicitly (e.g. referring to a particular experiment); 
% \item When submitted to arXiv, the supplementary will already included at the end of the paper.
% \end{itemize}
% % 
% To split the supplementary pages from the main paper, you can use \href{https://support.apple.com/en-ca/guide/preview/prvw11793/mac#:~:text=Delete%20a%20page%20from%20a,or%20choose%20Edit%20%3E%20Delete).}{Preview (on macOS)}, \href{https://www.adobe.com/acrobat/how-to/delete-pages-from-pdf.html#:~:text=Choose%20%E2%80%9CTools%E2%80%9D%20%3E%20%E2%80%9COrganize,or%20pages%20from%20the%20file.}{Adobe Acrobat} (on all OSs), as well as \href{https://superuser.com/questions/517986/is-it-possible-to-delete-some-pages-of-a-pdf-document}{command line tools}.
\section{Details About NextFlow}
\subsection{Architecture}
The NextFlow main Transformer is initialized from Qwen2.5-VL-7B~\cite{Qwen2.5-VL} and augmented with a newly introduced visual codebook and a revised logits-prediction head.

\subsection{Vision Generation}
For image synthesis, the model first emits the special token $\langle\text{boi}\rangle$ (begin-of-image) and then switches to full attention, operating in the style of VAR~\cite{var}. Scale-specific configurations are provided in Table~\ref{tab:patch_num_list}.

\begin{table*}[h]
% \tiny
\centering
\setlength{\tabcolsep}{2pt}
\captionsetup{skip=5pt} 
\caption{Related tokens corresponding to different prefix resolutions are provided. Note that the method also supports varying aspect ratios.}
\label{tab:patch_num_list}
\resizebox{.99\linewidth}{!}{
\begin{tabular}{l c c c c c c c c c c c c c c c c c c c}
\toprule
Resolution & \multicolumn{18}{c}{Related Schedule} \\
\midrule
64$\times$64 & (1,1) & (2,2) & (3,3) & (4,4) &  &  &  &  & & & & & & &  &  &  &  \\
128$\times$128 & (1,1) & (2,2) & (3,3) & (4,4) & (5,5) & (6,6) & (7,7) & (8,8) & & & & & & &  &  &  &  \\
256$\times$256 & (1,1) & (2,2) & (3,3) & (4,4) & (5,5) & (6,6) & (7,7) & (8,8) & (10,10) & (12,12) & (14,14) & (16,16) & & &  &  &  &  \\
512$\times$512 & (1,1) & (2,2) & (3,3) & (4,4) & (5,5) & (6,6) & (7,7) & (8,8) & (10,10) & (12,12) & (14,14) & (16,16) & (20,20) & (24,24) & (28,28) & (32,32) &  &  \\
1024$\times$1024 & (1,1) & (2,2) & (3,3) & (4,4) & (5,5) & (6,6) & (7,7) & (8,8) & (10,10) & (12,12) & (14,14) & (16,16) & (20,20) & (24,24) & (28,28) & (32,32) & (48,48) & (64,64) \\
\bottomrule
\end{tabular}
}
\end{table*}

\section{Captions}
\label{sec.prompts}
The detailed prompts for Fig. \ref{fig:ocr_compare} and Fig. \ref{fig:hps_visual} is shown in Tab. \ref{tab:long_prompts_hps} and Tab. \ref{tab:long_prompts_ocr}.

\section{Distribution Gap between the Training and Evaluation Datasets}
Our in-house training corpus exhibits a pronounced imbalance in the \textit{number of text-rendering regions} per sample: an overabundance of single-region cases and a long tail of images containing more than five regions. To mitigate this skew, we adopt a region-count–based filtering strategy for the training data. For analysis and visualization, we discretize the per-sample region count into six categories—1, 2, 3, 4, 5, and $>$5—using the $>$5 bin to summarize the long tail. The evaluation set follows a target profile restricted to bins 2–5 with probabilities 0.2, 0.3, 0.3, and 0.2, respectively. Fig.~\ref{fig:ocr-rendering-dataset} compares the empirical distributions for the pre-filter training set, the post-filter training set, and the evaluation set. Importantly, this filtering does not constitute evaluation hacking or test-set leakage; rather, it calibrates the training distribution to the intended task difficulty.

\begin{figure}[t]
  \centering
  \includegraphics[width=.82\linewidth]{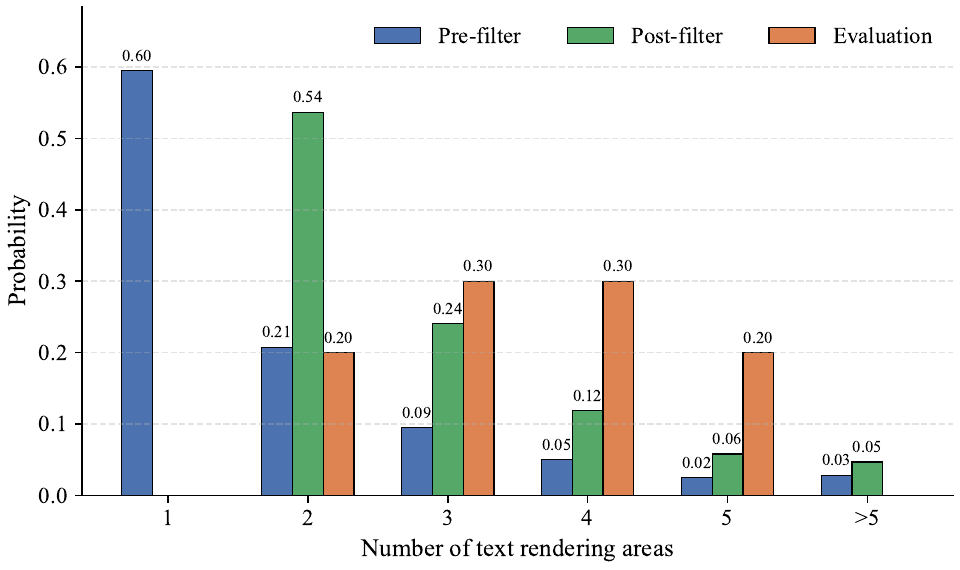}
  \caption{Distributions over the number of text-rendering regions (1, 2, 3, 4, 5, and $>$5) for the pre-filter training set, the post-filter training set, and the evaluation set. The evaluation distribution is constrained to bins 2–5 with probabilities 0.2, 0.3, 0.3, and 0.2. }
  \label{fig:ocr-rendering-dataset}
\end{figure}

% \section{More VMR Scales}
% To clarify our choice of $m$, we include experimental results for a prefix size of 64 as in Fig.~\ref{fig:more_prefix}. As shown, the optimal setting lies between 128 and 256, consistent with the findings in Subsection~\ref{subsec: ablation}.
% \begin{figure}[t]
%   \centering
%   \includegraphics[width=.8\linewidth]{fig/rl_prefix.pdf}
%    %\includegraphics[width=0.8\linewidth]{egfigure.eps}

%    \caption{Comparison of training curves between vanilla GRPO and GRPO with VMR across varying prefix scales.}
%    \label{fig:more_prefix}
% \end{figure}
\section{More Visual Results}
The visual samples from the HPSv3 evaluation dataset for each class are shown in Fig.~\ref{fig:more_visual} and Fig.~\ref{fig:more_visual_2}.

\begin{figure}[t]
  \centering
  \includegraphics[width=.82\linewidth]{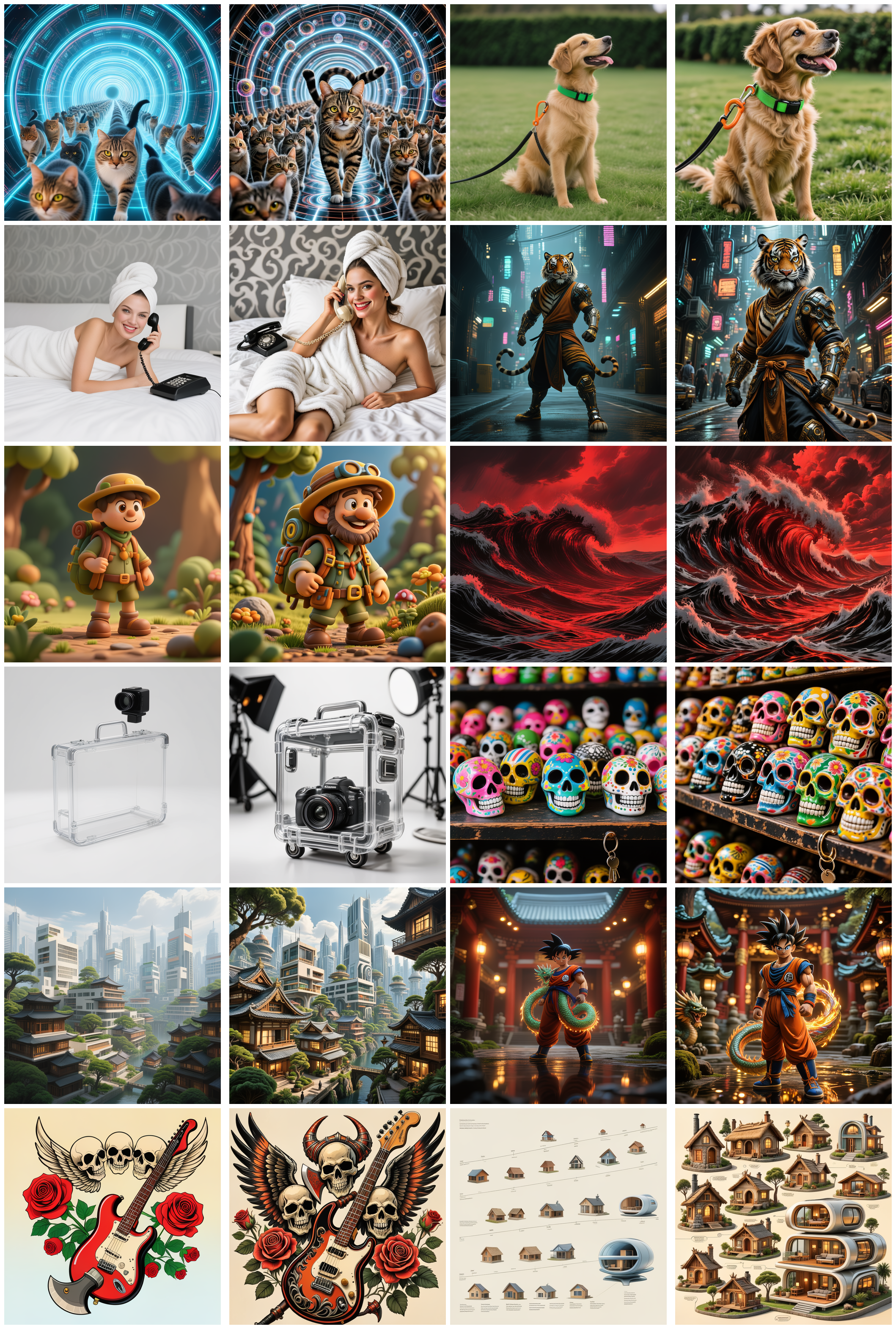}

   \caption{Random visual samples from the HPSv3 evaluation set by class. Each class includes two pairs; within each pair, the left image is before RL and the right image is after RL.}
   \label{fig:more_visual}
\end{figure}

\begin{figure}[t]
  \centering
  \includegraphics[width=.82\linewidth]{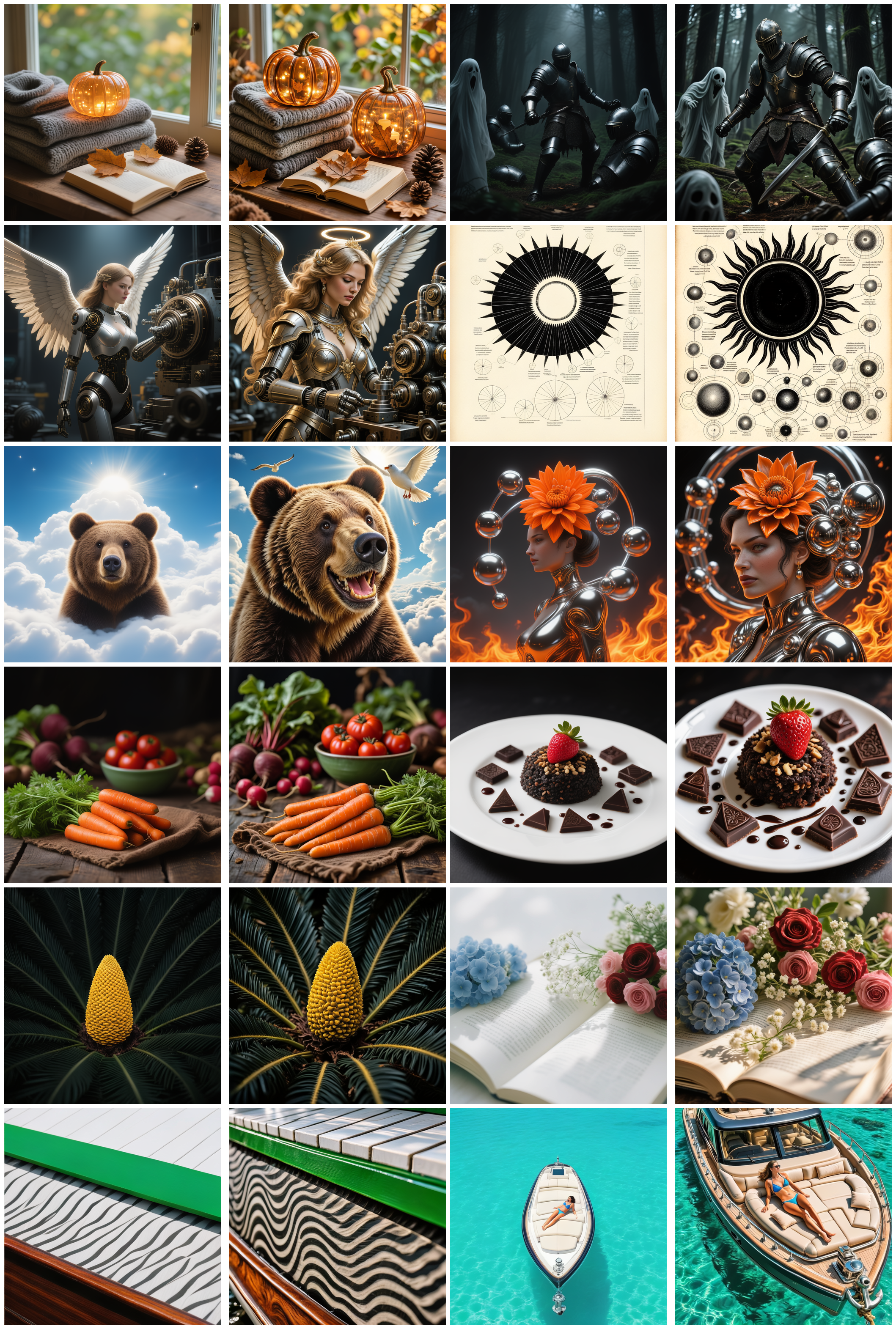}

   \caption{Random visual samples from the HPSv3 evaluation set by class. Each class includes two pairs; within each pair, the left image is before RL and the right image is after RL.}
   \label{fig:more_visual_2}
\end{figure}

\section{Proofs}
\subsection{Proof of Theorem~\ref{thm:reverse_kl_projection}}
\thmrklproj*
\begin{proof}
We first derive the trajectory-level variational identity, then
identify the global soft-optimal policy $\pi^\ast$ and its associated
$Q^\ast$, and finally show that the constrained optimal policy
$\pi^\dagger$ is a reverse-KL projection of $\pi^\ast$ at each state.

\textbf{1. Trajectory-level formulation of $J(\pi)$.}
Let a trajectory be
\begin{equation}
\tau = (\mathbf{s}_1,\mathbf{a}_1,\dots,\mathbf{s}_{T}),
\end{equation}
and denote by $p_{\pi}(\tau)$ the trajectory distribution induced by a policy
$\pi$ together with the environment dynamics. For a fixed
reference policy $\pi_{\mathrm{old}}$, let $p_{\mathrm{old}}(\tau)$ be the
corresponding trajectory distribution.

The KL-regularized objective is
\begin{align}
J(\pi)
&=
\mathbb{E}_{p_{\pi}(\tau)}\!\left[ R(\mathbf{s}_T) \right]
-
\eta\, \mathrm{KL}\!\left( p_{\pi}(\tau) \big\| p_{\mathrm{old}}(\tau) \right)
\\
&=
\mathbb{E}_{p_{\pi}(\tau)}\!\left[ R(\mathbf{s}_T) \right]
-
\eta\,
\mathbb{E}_{p_{\pi}(\tau)}\!\left[
\log \frac{p_{\pi}(\tau)}{p_{\mathrm{old}}(\tau)}
\right].
\end{align}
Using the factorization
\begin{equation}
p_{\pi}(\tau)
=
p(\mathbf{s}_1)\prod_{t=1}^{T-1}
\pi(\mathbf{a}_t\mid \mathbf{s}_t)\,
p(\mathbf{s}_{t+1}\mid \mathbf{s}_t,\mathbf{a}_t),
\end{equation}
and similarly for $p_{\mathrm{old}}(\tau)$ with $\pi_{\mathrm{old}}$, the
transition dynamics and initial-state distribution cancel in the ratio, giving
\begin{equation}
\log\frac{p_{\pi}(\tau)}{p_{\mathrm{old}}(\tau)}
=
\sum_{t=1}^{T-1}
\log\frac{\pi(\mathbf{a}_t\mid \mathbf{s}_t)}{\pi_{\mathrm{old}}(\mathbf{a}_t\mid \mathbf{s}_t)}.
\end{equation}
Thus
\begin{equation}
\label{eq:J_traj_form}
J(\pi)
=
\mathbb{E}_{p_{\pi}(\tau)}\!\left[
R(\mathbf{s}_T)
- \eta \sum_{t=1}^{T-1}
\log \frac{\pi(\mathbf{a}_t\mid \mathbf{s}_t)}{\pi_{\mathrm{old}}(\mathbf{a}_t\mid \mathbf{s}_t)}
\right].
\end{equation}

\textbf{2. Variational identity and global soft-optimal trajectory distribution.}
Define an unnormalized density over trajectories
\begin{equation}
\tilde{p}(\tau)
=
p_{\mathrm{old}}(\tau)
\exp\!\left(
\tfrac{1}{\eta} R(\mathbf{s}_T)
\right),
\end{equation}
with partition function
\begin{equation}
Z = \int \tilde{p}(\tau)\, d\tau
=
\mathbb{E}_{p_{\mathrm{old}}(\tau)}\!\left[\exp\!\left(\tfrac{1}{\eta} R(\mathbf{s}_T)\right)\right].
\end{equation}
Define the normalized distribution
\begin{equation}
p^\ast(\tau)
=
\frac{1}{Z}\tilde{p}(\tau)
=
\frac{1}{Z}
p_{\mathrm{old}}(\tau)\exp\!\left(\tfrac{1}{\eta} R(\mathbf{s}_T)\right).
\end{equation}

We now relate $J(\pi)$ to the KL divergence
$\mathrm{KL}(p_{\pi}(\tau)\,\|\,p^\ast(\tau))$.
Compute
\begin{align}
\mathrm{KL}\!\left(p_{\pi}(\tau)\,\big\|\,p^\ast(\tau)\right)
&=
\int p_{\pi}(\tau)
\log\frac{p_{\pi}(\tau)}{p^\ast(\tau)}\,d\tau
\\
&=
\int p_{\pi}(\tau)
\left[
\log\frac{p_{\pi}(\tau)}{p_{\mathrm{old}}(\tau)}
-
\tfrac{1}{\eta} R(\mathbf{s}_T)
+
\log Z
\right] d\tau
\\
&=
\mathbb{E}_{p_{\pi}}\!\left[
\log\frac{p_{\pi}(\tau)}{p_{\mathrm{old}}(\tau)}
\right]
-
\tfrac{1}{\eta}\mathbb{E}_{p_{\pi}}\!\left[R(\mathbf{s}_T)\right]
+
\log Z.
\end{align}
Multiplying both sides by $\eta$ gives
\begin{equation}
\label{eq:eta_KL_relation}
\eta\,\mathrm{KL}\!\left(p_{\pi}(\tau)\,\big\|\,p^\ast(\tau)\right)
=
\eta\,\mathrm{KL}\!\left(p_{\pi}(\tau)\,\big\|\,p_{\mathrm{old}}(\tau)\right)
-
\mathbb{E}_{p_{\pi}}[R(\mathbf{s}_T)]
+
\eta\log Z.
\end{equation}
Rearranging \eqref{eq:eta_KL_relation} yields the trajectory-level
variational identity:
\begin{equation}
\label{eq:traj_variational_identity}
\mathbb{E}_{p_{\pi}}[R(\mathbf{s}_T)]
-
\eta\, \mathrm{KL}\!\left( p_{\pi}(\tau) \big\| p_{\mathrm{old}}(\tau)\right)
=
\eta \log Z
-
\eta\,
\mathrm{KL}\!\left( p_{\pi}(\tau) \big\| p^\ast(\tau)\right).
\end{equation}
Thus
\begin{equation}
J(\pi)
=
\eta \log Z
-
\eta\,\mathrm{KL}\!\left( p_{\pi}(\tau) \big\| p^\ast(\tau)\right),
\end{equation}
and $J(\pi)$ is maximized exactly when $p_{\pi}(\tau) = p^\ast(\tau)$.

\textbf{3. Factorization of $p^\ast(\tau)$ and the soft-optimal policy.}
Since $p_{\mathrm{old}}(\tau)$ is induced by the Markov policy
$\pi_{\mathrm{old}}$ and the environment dynamics, it factorizes as
\begin{equation}
p_{\mathrm{old}}(\tau)
=
p(\mathbf{s}_1)
\prod_{t=1}^{T-1}
\pi_{\mathrm{old}}(\mathbf{a}_t\mid \mathbf{s}_t)\,
p(\mathbf{s}_{t+1}\mid \mathbf{s}_t,\mathbf{a}_t).
\end{equation}
Hence
\begin{equation}
\label{eq:pstar_factor}
p^\ast(\tau)
=
\frac{1}{Z}
p(\mathbf{s}_1)
\prod_{t=1}^{T-1}
\pi_{\mathrm{old}}(\mathbf{a}_t\mid \mathbf{s}_t)\,
p(\mathbf{s}_{t+1}\mid \mathbf{s}_t,\mathbf{a}_t)\,
\exp\!\left(\tfrac{1}{\eta} R(\mathbf{s}_T)\right).
\end{equation}
This defines a Markov trajectory distribution; thus there exists a policy
$\pi^\ast$ such that
\begin{equation}
p^\ast(\tau)
=
p(\mathbf{s}_1)
\prod_{t=1}^{T-1}
\pi^\ast(\mathbf{a}_t\mid \mathbf{s}_t)\,
p(\mathbf{s}_{t+1}\mid \mathbf{s}_t,\mathbf{a}_t).
\end{equation}

We now compute $\pi^\ast(\mathbf{a}_t\mid \mathbf{s}_t)$ explicitly.
Fix a time $t$ and a state $\mathbf{s}_t$. Using \eqref{eq:pstar_factor},
the conditional distribution over $\mathbf{a}_t$ given $\mathbf{s}_t$ under
$p^\ast$ is proportional to the joint density over all trajectories sharing
$(\mathbf{s}_t,\mathbf{a}_t)$:
\begin{align}
\pi^\ast(\mathbf{a}_t\mid \mathbf{s}_t)
&=
p^\ast(\mathbf{a}_t\mid \mathbf{s}_t)
\propto
\sum_{\text{trajectories consistent with }(\mathbf{s}_t,\mathbf{a}_t)}
p^\ast(\tau)
\\
&\propto
\sum_{\text{futures}} p_{\mathrm{old}}(\tau)\,
\exp\!\left(\tfrac{1}{\eta} R(\mathbf{s}_T)\right),
\end{align}
where the sum is over the future part of the trajectory (from $t$ onward),
and the past part $\mathbf{s}_{1:t-1},\mathbf{a}_{1:t-1}$ only contributes a
constant factor w.r.t.\ $\mathbf{a}_t$. More precisely, for fixed
$\mathbf{s}_t$ we have
\begin{align}
\pi^\ast(\mathbf{a}_t\mid \mathbf{s}_t)
&\propto
\pi_{\mathrm{old}}(\mathbf{a}_t\mid \mathbf{s}_t)
\sum_{\mathbf{s}_{t+1:T},\mathbf{a}_{t+1:T-1}}
\left[
\prod_{k=t}^{T-1}
p(\mathbf{s}_{k+1}\mid \mathbf{s}_k,\mathbf{a}_k)
\pi_{\mathrm{old}}(\mathbf{a}_{k+1}\mid \mathbf{s}_{k+1})
\right]
\exp\!\left(\tfrac{1}{\eta} R(\mathbf{s}_T)\right)
\\
&=
\pi_{\mathrm{old}}(\mathbf{a}_t\mid \mathbf{s}_t)\,
\mathbb{E}_{\pi_{\mathrm{old}}}
\left[
\exp\!\left(\tfrac{1}{\eta} R(\mathbf{s}_T)\right)
\;\middle|\;
\mathbf{s}_t,\mathbf{a}_t
\right].
\end{align}
Thus we obtain the \textit{general optimal solution}
\begin{equation}
% \label{general optimal solution}
\pi^\ast(\mathbf{a}_t \mid \mathbf{s}_t)
\propto
\pi_{\mathrm{old}}(\mathbf{a}_t \mid \mathbf{s}_t)
\exp\left(\frac{1}{\eta} Q^\ast(\mathbf{s}_t, \mathbf{a}_t)\right),
\end{equation}
where we have defined the optimal soft $Q$-function as
\begin{equation}
% \label{easy_optimal_Q}
Q^\ast(\mathbf{s}_t, \mathbf{a}_t)
=
\eta \ln \mathbb{E}_{\pi_{\mathrm{old}}}
\left[
\exp\left(\frac{1}{\eta} R(\mathbf{s}_T)\right)
\;\middle|\;
\mathbf{s}_t, \mathbf{a}_t
\right].
\end{equation}
Equations \eqref{general optimal solution} and \eqref{easy_optimal_Q} fully
characterize the global soft-optimal policy $\pi^\ast$.

\textbf{4. Constrained optimum as trajectory-level reverse-KL projection.}
By the variational identity \eqref{eq:traj_variational_identity}, for any
policy $\pi$,
\begin{equation}
J(\pi)
=
\eta \log Z
-
\eta\,
\mathrm{KL}\!\left( p_{\pi}(\tau) \big\| p^\ast(\tau)\right).
\end{equation}
Therefore, when we restrict $\pi$ to the constraint set
$\mathcal{M}_{\pi}$ (Definition~\ref{def:var_constrained}), the optimal
policy
\begin{equation}
\pi^\dagger
=\arg\max_{\pi\in\mathcal{M}_{\pi}}J(\pi)
\end{equation}
is equivalently given by
\begin{equation}
\pi^\dagger
=
\arg\min_{\pi\in\mathcal{M}_{\pi}}
\mathrm{KL}\!\left(p_\pi(\tau)\,\big\|\,p^\ast(\tau)\right).
\end{equation}

\textbf{5. Decomposition of trajectory KL into per-state KLs.}
Both $p_\pi(\tau)$ and $p^\ast(\tau)$ factorize according to the same dynamics:
\begin{align}
p_\pi(\tau)
&=
p(\mathbf{s}_1)
\prod_{t=1}^{T-1}
\pi(\mathbf{a}_t\mid \mathbf{s}_t)\,
p(\mathbf{s}_{t+1}\mid \mathbf{s}_t,\mathbf{a}_t),\\
p^\ast(\tau)
&=
p(\mathbf{s}_1)
\prod_{t=1}^{T-1}
\pi^\ast(\mathbf{a}_t\mid \mathbf{s}_t)\,
p(\mathbf{s}_{t+1}\mid \mathbf{s}_t,\mathbf{a}_t).
\end{align}
Hence the KL divergence simplifies to
\begin{align}
\mathrm{KL}\!\left(p_\pi(\tau)\,\big\|\,p^\ast(\tau)\right)
&=
\mathbb{E}_{p_\pi(\tau)}\!\left[
\log\frac{p_\pi(\tau)}{p^\ast(\tau)}
\right]
\\
&=
\mathbb{E}_{p_\pi(\tau)}\!\left[
\sum_{t=1}^{T-1}
\log\frac{\pi(\mathbf{a}_t\mid \mathbf{s}_t)}{\pi^\ast(\mathbf{a}_t\mid \mathbf{s}_t)}
\right]
\\
&=
\sum_{t=1}^{T-1}
\mathbb{E}_{\mathbf{s}_t\sim d_\pi}\!
\left[
\mathrm{KL}\!\left(\pi(\cdot\mid \mathbf{s}_t)\,\big\|\,\pi^\ast(\cdot\mid \mathbf{s}_t)\right)
\right],
\end{align}
where $d_\pi$ denotes the (discount-free) state visitation distribution under
$\pi$ at time $t$. The key point is that the dynamics and initial-state
distribution cancel and only the action distributions appear inside the KL.

\textbf{6. Factorized constraint and per-state reverse-KL projection.}
The constraint set $\mathcal{M}_{\pi}$ is assumed to factorize across
states:
\begin{equation}
\mathcal{M}_{\pi}
=
\prod_{\mathbf{s}}
\mathcal{M}_{\pi}(\mathbf{s}),
\end{equation}
where $\mathcal{M}_{\pi}(\mathbf{s})$ is the feasible set of action
distributions at state $\mathbf{s}$. Because of this product structure,
choosing $\pi\in\mathcal{M}_{\pi}$ amounts to choosing independently
each $\pi(\cdot\mid \mathbf{s})\in\mathcal{M}_{\pi}(\mathbf{s})$.

Since
\begin{equation}
\mathrm{KL}\!\left(p_\pi(\tau)\,\big\|\,p^\ast(\tau)\right)
=
\sum_{t}
\mathbb{E}_{\mathbf{s}_t\sim d_\pi}
\Big[
\mathrm{KL}\big(\pi(\cdot\mid \mathbf{s}_t)\,\big\|\,\pi^\ast(\cdot\mid \mathbf{s}_t)\big)
\Big],
\end{equation}
minimizing this KL over the product set
$\mathcal{M}_{\pi}=\prod_{\mathbf{s}}\mathcal{M}_{\pi}(\mathbf{s})$
decouples into independent minimizations at each state:
\begin{equation}
\pi^\dagger(\cdot\mid \mathbf{s}_t)
=
\arg\min_{\pi(\cdot\mid \mathbf{s}_t)\in\mathcal{M}_{\pi}(\mathbf{s}_t)}
\mathrm{KL}\!\big(\pi(\cdot\mid \mathbf{s}_t)\,\big\|\,\pi^\ast(\cdot\mid \mathbf{s}_t)\big).
\end{equation}

\textbf{7. Conclusion.}
Combining the above steps, we conclude that the constrained optimal policy
$\pi^\dagger$ is obtained by, at each state $\mathbf{s}_t$, projecting the
global soft-optimal policy $\pi^\ast$ onto the feasible action-distribution
set $\mathcal{M}_{\pi}(\mathbf{s}_t)$ in the sense of reverse KL:
\begin{equation}
\pi^\dagger(\cdot\mid \mathbf{s}_t)
=
\arg\min_{\pi\in\mathcal{M}_{\pi}(\mathbf{s}_t)}
\mathrm{KL}\!\big(\pi(\cdot\mid \mathbf{s}_t)\,\|\,\pi^\ast(\cdot\mid \mathbf{s}_t)\big).
\end{equation}
This proves the theorem.
\end{proof}
\subsection{Proof of Theorem~\ref{thm:two_stage_invariance_var}}
\thmtsinvar*
\begin{proof}
Write the full objective by splitting the action sequence into prefix and
suffix:
\begin{equation}
J(\pi)
=
\underbrace{
\mathbb{E}\big[ R(s_T)\,\big|\, s_m,\pi_{m:T-1}\big]
-
\eta\,
\mathrm{KL}\!\left(\pi_{m:T-1}\,\big\|\, \pi_{\mathrm{old},\,m:T-1}\right)
}_{\triangleq\;J_{\mathrm{suffix}}(s_m,\pi_{m:T-1})}
+
\underbrace{
-\eta\,
\mathrm{KL}\!\left(\pi_{1:m-1}\,\big\|\, \pi_{\mathrm{old},\,1:m-1}\right)
}_{\text{prefix KL}}
,
\end{equation}
and take the outer expectation over \(s_m\) induced by the prefix policy
\(\pi_{1:m-1}\). For any fixed \(s_m\), the inner maximization over suffix
policies is exactly the standard soft-control problem whose value is
\(V_m^\ast(s_m)\). Therefore
\begin{equation}
\max_{\pi}\, J(\pi)
=
\max_{\pi_{1:m-1}}
\Big\{
\mathbb{E}\big[ V_m^\ast(s_m)\,\big|\, \pi_{1:m-1}\big]
-
\eta\,
\mathrm{KL}\!\left(\pi_{1:m-1}\,\big\|\, \pi_{\mathrm{old},\,1:m-1}\right)
\Big\},
\end{equation}
and the maximizer is obtained by (i) choosing the suffix soft-optimal
\(\pi^{\ast}_{m:T-1}\) for each \(s_m\), and (ii) choosing the prefix policy
\(\pi^{\ast}_{1:m-1}\) that maximizes the soft value of \(s_m\).
Concatenation yields the unique full-horizon maximizer \(\pi^\ast\).
\end{proof}

{
\footnotesize % 设置区域字体大小
\onecolumn
\begin{longtable}{p{0.05\textwidth} p{0.9\textwidth}}
    \caption{Detailed prompts used in Fig. \ref{fig:ocr_compare}.} \label{tab:long_prompts_ocr} \\
    \toprule
    \textbf{ID} & \textbf{Prompt Content} \\
    \midrule
    \endfirsthead
    
    \multicolumn{2}{c}%
    {{\bfseries \tablename\ \thetable{} -- continued from previous page}} \\
    \toprule
    \textbf{ID} & \textbf{Prompt Content} \\
    \midrule
    \endhead
    
    \midrule
    \multicolumn{2}{r}{{Continued on next page}} \\
    \bottomrule
    \endfoot
    
    \bottomrule
    \endlastfoot

    1 & Six illuminated letters ('A', 'B', 'C', 'N', 'O', 'Y') in two rows on a dark blue background, outlined with white bulbs glowing blue. \\
    \midrule
    2 & The image is a shield-shaped graphic with a black border. Inside the shield, there are three horizontal stripes: red at the top, white in the middle, and blue at the bottom. The red stripe contains white text that reads, ``Wouldn't you rather...'' The white stripe in the middle has bold black text that says, ``VOTE.'' And the blue stripe at the bottom features white text that reads, ``By Mail.'' The overall design suggests a campaign or initiative encouraging people to vote by mail. \\
    \midrule
    3 & The image features two small, round, metallic tin containers placed on a dark, textured fabric background. The lids of the tins are slightly open, revealing the contents inside. The tins have a vintage or antique appearance, with one lid displaying an engraved design. The text ``M-BOX 2.0'' is prominently overlaid on the image in a bold, yellow font, while ``by: Jimmy-Fan'' is written in a smaller, yellow font below it. The overall composition suggests a promotional or product image for the tins, textitasizing their design and branding. \\
    \midrule
    4 & The image depicts a serene nighttime camping scene in a forest. The atmosphere is illuminated by a soft, glowing light emanating from within a tent, casting a warm and inviting glow on the surrounding area. The tents are set up on a forest floor covered with pine needles and leaves. The trees are tall and dense, creating a canopy that stretches out into the night sky. The overall mood of the image is tranquil and adventurous, capturing the essence of an outdoor camping experience. In the upper right corner, there is text in Spanish that reads ``DESCUBRE LO QUE ESCONDEN NUESTRO CAMPING,'' and in the bottom right corner, the words ``UPBAN KIDS'' are displayed, with ``UPBAN'' in white and ``KIDS'' in green. \\
    \midrule
    5 & The image shows a cardboard sign with handwritten text in French. The text reads ``L'ARANCAEY FIMMNA,'' with ``L'ARANCAEY'' written in purple and ``FIMMNA'' in black. The sign appears to be placed against a wall, and there are some items and containers visible in the background. \\
    \midrule
    6 & The image features a cartoon-style illustration of a person standing and holding a fishing rod. The person appears to be a stylized representation of a man with light skin, wearing a white shirt, a red tie, and dark pants. The fishing rod is angled upwards, and at the end of the line, there is a white fish with an upward-pointing arrow, suggesting the fish has been caught. Above the person, there is a red banner with the text ``FISHING CHALLENGE'' in white letters. To the right of the person, the word ``GOTCHA!'' is written in white text, indicating the successful catch of the fish. The background is a solid light green color. \\
    \midrule
    7 & The image features four packets of snacks against an orange background. Each packet has a distinct design and is labeled with different types of snacks. The first packet on the left has a black and orange swirl pattern and is labeled ``Apps Way Crisps.'' The second packet has a yellow and black crisscross pattern and is labeled ``Apps Wheat Crisps.'' The third packet has black diagonal stripes and is labeled ``Apps Potato Crisps.'' The fourth packet has a black polka dot pattern and is labeled ``Apps Cheese Balls.'' The design elements and text on each packet are consistent, suggesting they are part of a branded product line. \\
    \midrule
    8 & The image shows a piece of paper with handwritten notes on it, placed on a wooden surface. The text writes: `VAR RL Done Right'. The handwriting is informal and includes some slang or abbreviations. The paper seems to be printed with a grid pattern, suggesting it might be a printed form or a piece of paper with a pre-defined layout. \\
\end{longtable}
}

{
\footnotesize % 设置区域字体大小
\onecolumn
\begin{longtable}{p{0.05\textwidth} p{0.9\textwidth}}
    \caption{Detailed prompts used in Fig. \ref{fig:hps_visual}.} \label{tab:long_prompts_hps} \\
    \toprule
    \textbf{ID} & \textbf{Prompt Content} \\
    \midrule
    \endfirsthead
    
    \multicolumn{2}{c}%
    {{\bfseries \tablename\ \thetable{} -- continued from previous page}} \\
    \toprule
    \textbf{ID} & \textbf{Prompt Content} \\
    \midrule
    \endhead
    
    \midrule
    \multicolumn{2}{r}{{Continued on next page}} \\
    \bottomrule
    \endfoot
    
    \bottomrule
    \endlastfoot

    1 & A cheerful, anthropomorphic squirrel stands upright against a solid light green background. The squirrel has a rich brown coat with a lighter cream-colored underbelly, muzzle, and inner ears, which are pink at the tips. Its large, expressive black eyes are wide and friendly, with a small, triangular nose and a wide smile showing two prominent front teeth. The squirrel's bushy tail is curled in a spiral at the end, with a gradient from dark brown at the base to a lighter cream color at the tip. Around its neck, it wears a vibrant flower lei composed of small, colorful blossoms in red, yellow, green, purple, and blue. The squirrel holds a large, smooth, light brown egg in both front paws, which are positioned close to its chest. Its posture is upright, with its hind legs planted firmly on the ground, and it casts a soft shadow beneath it on the green background. \\
    \midrule
    2 & An anime-style illustration depicts a muscular, metallic tiger with sharp, angular features, standing on a rooftop. The tiger is in a dynamic pose, gripping a sleek, red electric guitar, and its mouth is open wide as if caught in the midst of a powerful roar or song. Above the tiger, a bright spotlight casts a dramatic beam of light, illuminating the scene and creating stark shadows on the surrounding rooftop features. \\
    \midrule
    3 & A surreal figure appears to be sculpted from intertwining tendrils of gray smoke and whirling flurries of snow, giving the impression of a man caught in a blizzard. In one hand, this ethereal being holds what looks to be a gateway to the cosmos, depicted in a photorealistic manner with vibrant nebulae and star clusters visible within its confines. The entire scene is a highly detailed octane render, showcasing sharp contrasts and the interplay of light and shadow that imbues the image with a sens. \\
    \midrule
    4 & An ornate royal carriage, painted in deep red with golden trim, stands prominently against a landscape blanketed in pristine snow. Behind it, the silhouettes of tall pine trees dusted with white can be discerned through the soft haze of a winter's day. In front of the carriage, the snow-covered ground glistens under the subtle light of the afternoon sun. \\
    \midrule
    5 & a detailed 17th-century Dutch Baroque painting depicting a chestnut horse standing amidst a vibrant field of tulips and daisies. The horse's mane and tail are elegantly captured by the artist, flowing with the gentle breeze that seems to animate the scene. In the background, a traditional windmill sits under a partly cloudy sky, completing this pastoral landscape. \\
    \midrule
    6 & A photograph with a standard lens style. The subject is a man standing against a background that appears to be covered in splattered paint. The man has short, dark hair and a light beard. He is wearing a long-sleeved, button-up shirt that is also splattered with paint, suggesting he might have been involved in an artistic or creative activity. The shirt is a light color, possibly beige or off-white, and has two chest pockets. The man's expression is serious and contemplative. The background is primarily white with red and black paint splatters scattered across it. The lighting in the photograph is soft and even, highlighting the details of the man's face and the texture of his shirt. The overall mood of the image is artistic and somewhat moody, with a focus on the subject's expression and the abstract splatters in the background. \\
    \midrule
    7 & Two young girls stand on a lush, green grassy field, both dressed in white lace sleeveless dresses adorned with a purple satin ribbon tied in a bow at the waist. The girl in the foreground has shoulder-length, wavy blonde hair and wears a floral crown featuring white and purple flowers, including roses and possibly lisianthus, with small greenery accents. She holds a small bouquet composed of white flowers, purple blooms, and clusters of dark red berries, with some green foliage. Behind her, the second girl, also with blonde hair (slightly lighter and more tousled), wears an identical floral crown and dress. She holds a bouquet of vibrant pink flowers, likely carnations, with green stems. Both girls are smiling brightly, their eyes a striking blue, and the wind gently lifts strands of their hair. The background is a soft, out-of-focus expanse of green grass, dotted with tiny white flowers, creating a serene, natural setting. \\
    \midrule
    8 & a breathtaking photograph capturing the vibrant hues of a sunset with streaks of pink and orange painting the sky behind the majestic Grand Canyon. The canyon's intricate rock formations are silhouetted against the illuminated backdrop, showcasing the deep crevices and towering spires. In the foreground, the Colorado River can be glimpsed winding its way through the ancient geological marvel. \\
\end{longtable}
}

\end{CJK*}
\end{document}